\documentclass[letterpaper, 10 pt, conference]{ieeeconf}  

 \IEEEoverridecommandlockouts                              

\usepackage{amsmath,amssymb,amsthm,mathtools}
\usepackage{bm}
\usepackage{hyperref}
\hypersetup{colorlinks=true,linkcolor=blue,citecolor=blue,urlcolor=blue}

\usepackage[font=small, labelfont=bf, skip=4pt]{caption}
\newtheorem{assumption}{Assumption}

\newtheorem{lemma}{Lemma}
\newtheorem{theorem}{Theorem}

\newtheorem{corollary}{Corollary}
\theoremstyle{remark}
\newtheorem{remark}{Remark}
\newtheorem{problem}{Problem}
\usepackage{pgfplots}
\usepackage{pgfplotstable}
\usepackage{filecontents}
\usepgfplotslibrary{fillbetween}
\usepackage{multirow}
\pgfplotsset{compat=newest}

\providecommand{\norm}[1]{\left\lVert #1 \right\rVert}

\providecommand{\eigmax}{\lambda_{\max}}

\DeclareMathOperator{\sym}{sym}

\newcommand{\bx}{\boldsymbol{x}}
\newcommand{\bu}{\boldsymbol{u}}

\newcommand{\bo}{\boldsymbol{y}}
\newcommand{\ours}{\text{ContractionPPO} }

\usepackage{algorithm}
\usepackage{algorithmic}
\usepackage{tikz}                                  
\usetikzlibrary{
  positioning,            
  arrows.meta,            
  shapes, 
  calc,
  fit
}
\usepackage{xcolor}
\definecolor{mylightblue}{RGB}{174, 199, 232}
\definecolor{mylightorange}{RGB}{255, 187, 120} 
\definecolor{mylightgreen}{RGB}{152, 223, 138} 
\definecolor{mylightred}{RGB}{255, 152, 150}   
\definecolor{mylightpurple}{RGB}{197, 176, 213} 
\definecolor{mylightbrown}{RGB}{196, 156, 148} 
\definecolor{mylightpink}{RGB}{247, 182, 210}  
\definecolor{mylightgray}{RGB}{199, 199, 199}  
\definecolor{mylightolive}{RGB}{219, 219, 141} 
\definecolor{mylightcyan}{RGB}{158, 218, 229}  
\usepackage{graphicx}
\usepackage[caption=false,font=footnotesize]{subfig}

\usepackage[compress]{cite}
\usetikzlibrary{backgrounds}

\title{\LARGE \bf
ContractionPPO: Certified Reinforcement Learning via Differentiable Contraction Layers
}

\author{Vrushabh Zinage$^{1,\star}$, Narek Harutyunyan$^{2,\star}$,  Eric Verheyden$^{1,\star}$, Fred Y. Hadaegh$^{1}$, Soon-Jo Chung$^{1}$
\thanks{$^\star$ Equal contributions.
$^{1}$Vrushabh Zinage, Eric Verheyden, Fred Hadaegh, and Soon-Jo Chung are with the California Institute of Technology
$^{2}$Narek Harutyunyan (currently with Brown University) was a Caltech SURF researcher. 
This research is funded by the Defense Advanced Research Projects Agency Learning Introspective Control (DARPA LINC).
         {\tt\small }}
}

\begin{document}

\maketitle
\bibliographystyle{IEEEtran} 



\begin{abstract}
Legged locomotion in unstructured environments demands not only high-performance control policies but also formal guarantees to ensure robustness under perturbations. Control methods often require carefully designed reference trajectories, which are challenging to construct in high-dimensional, contact-rich systems such as quadruped robots. In contrast, Reinforcement Learning (RL) directly learns policies that implicitly generate motion, and uniquely benefits from access to privileged information, such as full state and dynamics during training, that is not available at deployment. We present ContractionPPO, a framework for certified robust planning and control of legged robots by augmenting Proximal Policy Optimization (PPO) RL with a state-dependent contraction metric layer. This approach enables the policy to maximize performance while simultaneously producing a contraction metric that certifies incremental exponential stability of the simulated closed-loop system. The metric is parameterized as a Lipschitz neural network and trained jointly with the policy, either in parallel or as an auxiliary head of the PPO backbone. While the contraction metric is not deployed during real-world execution, we derive upper bounds on the worst-case contraction rate and show that these bounds ensure the learned contraction metric generalizes from simulation to real-world deployment. Our hardware experiments on quadruped locomotion demonstrate that ContractionPPO enables robust, certifiably stable control even under strong external perturbations. Videos of experiments are available at \href{https://contractionppo.github.io/}{\textbf{https://contractionppo.github.io/}}.

\end{abstract}
\section{Introduction}

The development of robust and high-performance control policies is a central challenge in legged robotics, where the nonlinear dynamics and real-world uncertainties demand both adaptability and formal guarantees such as safety or stability. 
Reinforcement Learning (RL), and in particular, Proximal Policy Optimization (PPO) \cite{schulman2017proximal_ppo}, has become a popular approach for training locomotion policies that can exploit the full potential of modern robot hardware. PPO or in general RL-based controllers have demonstrated impressive results on a range of quadruped platforms in simulation and the real-world, enabling agile and versatile behaviors beyond the reach of classical model-based controllers \cite{katz2019mini_ref1,hutter2016anymal_ref2,jacoff2023taking_ref3}. However, RL relies on access to privileged information through simulations, such as full system state or dynamics, which is unavailable at deployment time. This enables more effective reward shaping, learning signals, and auxiliary objectives during training. These advantages can be leveraged to learn structured internal representations or certify behavior, even when the final deployed policy relies solely on partial observations.

Despite these advances, a key limitation remains that RL policies, while powerful, typically lack formal guarantees on stability or robustness \cite{kumar2022adapting_rma_limit1,choi2023learning_limit2,luo2024moral_limit3,fu2023deep,yoon2024learning}. 
\begin{figure}[]
    \centering
\includegraphics[width=1\linewidth]{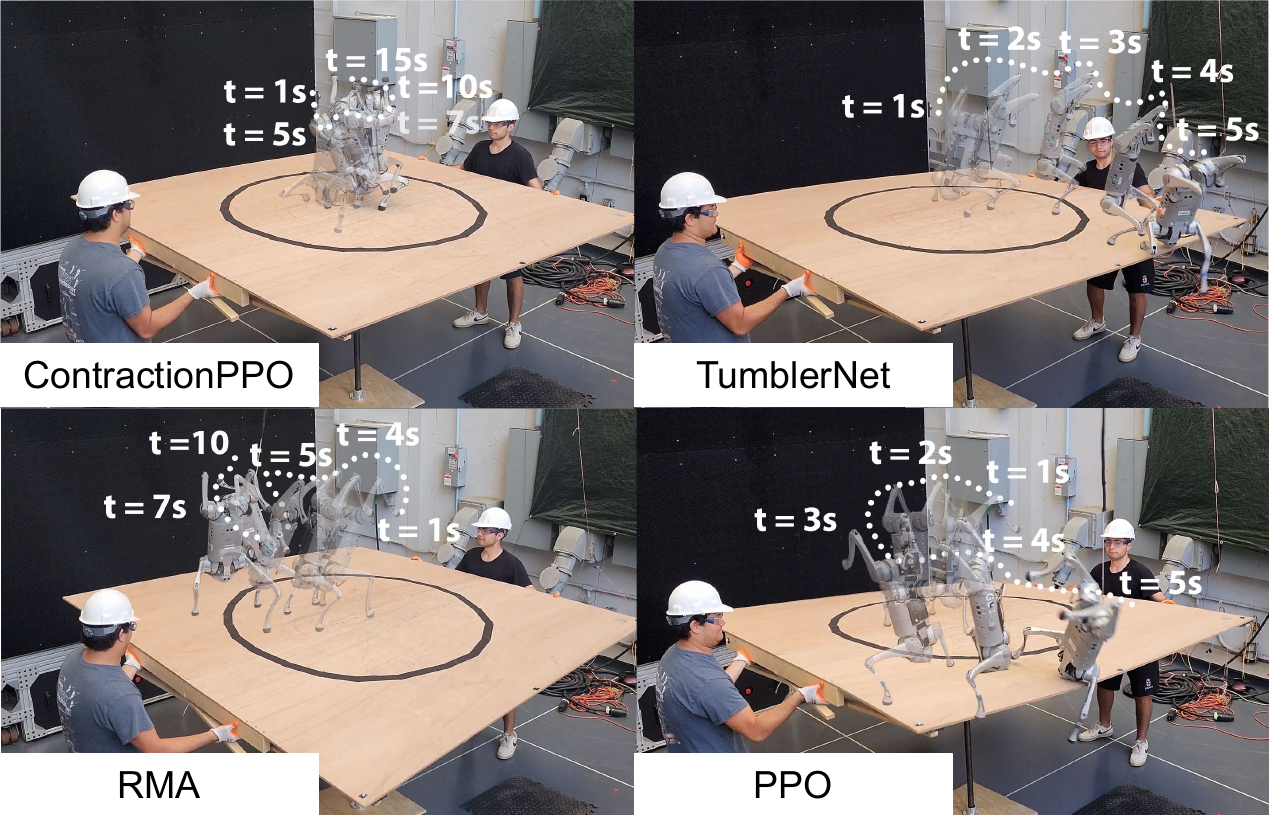}
    \caption{\footnotesize Comparison of trajectories for quadruped handstand using ContractionPPO (ours), TumblerNet \cite{xiao2025learning}, RMA \cite{kumar2022adapting_rma_limit1} and PPO \cite{schulman2017proximal_ppo} where robots (for all baseline algorithms) were trained with identical reward functions that encourage remaining close to their initial position. While PPO, TumblerNet and RMA leaves the region and ultimately fails to remain on platform, \ours guarantees that the robot lie inside the circle (black).
    }
    \label{fig:}
    \vspace{-0.5cm}
\end{figure}
In this paper, we derive a method of using nonlinear stability theory in simulation-based RL training. In particular, contraction theory~\cite{lohmiller1998contraction,tsukamoto2020neural_contraction,tsukamoto2021contraction_tutorial} provides necessary and sufficient conditions for incremental stability of nonlinear closed-loop systems via state-dependent contraction metrics. Unlike classical Lyapunov approaches, contraction analysis certifies that all system trajectories converge exponentially toward one another, offering a powerful and general stability guarantee. However, applying contraction theory in its classical form to high-dimensional nonlinear systems like legged robots is often challenging as it requires hand-designed metrics and solving matrix inequalities~\cite{tsukamoto2021contraction_tutorial}. Traditionally, a further challenge is the need to specify desired trajectories or setpoints for the system to track. In high-dimensional, contact-rich systems like legged robots, constructing such trajectories analytically or through heuristics can be nontrivial or even infeasible.

Building on these insights, we present a framework ContractionPPO, that integrates seamlessly with the PPO planning pipeline i.e., the contraction metric network layer can be trained either in parallel to the policy or as an additional network head. Consequently, this yields explicit and computable lower bounds on the worst-case contraction rate, directly in terms of the network Lipschitz constants and system dynamics. We take advantage of closed-loop simulation used for RL training by using privileged full state information to train the contraction metric, while the control policy itself remains observation based, preserving real-world deployability. The result is a certified, robust locomotion controller that offers not only strong empirical performance but also formal, verifiable guarantees of stability/safety even in the presence of observation noise, modeling error, and external disturbances.
To summarize, the contributions of this work are:
\begin{enumerate}
\item We introduce a differentiable contraction metric layer for PPO, enabling end-to-end learning of both a high-performance control policy and a stability certificate that guarantees incremental exponential stability for robust legged locomotion, taking a significant step from conventional RL toward certified RL.
\item We leverage privileged state information available during training to learn the contraction metric i.e., a state-dependent, positive-definite matrix function that certifies exponential stability of state trajectories that converge toward one another. At the same time, we retain a deployable, observation-based policy. We also derive explicit bounds on the contraction rate using Lipschitz-constrained networks, ensuring the learned stability certificate remains valid from simulation to real-world deployment.
\item Without being trained under extreme perturbations such as strong external wind disturbances, \ours consistently preserves incremental stability at test time. We validate this through simulation and hardware experiments, showing certifiably robust performance and safe recovery under severe domain shifts and disturbances.
\end{enumerate}

The remainder of this paper is organized as follows. Section~\ref{sec:related_work} discusses the related work followed by preliminaries and problem statement in Section~\ref{sec:prelim_and_problem}. Section~\ref{sec:proposed_approach} discusses the proposed approach followed by results in Section~\ref{sec:results}. Finally, concluding remarks are made in Section~\ref{sec:conclusion}.

\begin{figure*}[t]
    \centering
    \begin{minipage}[t]{0.68\textwidth}
        \centering
          \begin{tikzpicture}[scale=0.51, 
      >=Latex,
      node distance=0.8cm and 1cm,
      every node/.style={
        rectangle,transform shape, rounded corners, draw,
        align=center,
        inner sep=2pt, outer sep=0pt,
        minimum width=2.5cm, minimum height=0.8cm
      }
    ]
    \tikzset{
      block1/.style={fill=red!30},
      block3/.style={fill=blue!30},
      block4/.style={fill=orange!30},
      block5/.style={fill=purple!30},
      block6/.style={fill=cyan!30},
      block7/.style={fill=yellow!30},
      block8/.style={fill=lime!30},
      block9/.style={fill=pink!30},
      block11/.style={fill=teal!30},
      block12/.style={fill=violet!30},
      block13/.style={fill=brown!30},
      block14/.style={fill=magenta!30}
    }
\node[draw=none, inner sep=0pt, outer sep=0pt, rectangle, minimum size=0pt] at (-1.3,1.6) {\large PPO Layer};
\node[draw=none, inner sep=0pt, outer sep=0pt, rectangle, minimum size=0pt] at (-0.5,-2.7) {\large Contraction Layer};

        \node[ draw, fill=mylightred, rounded corners, minimum width=3.7cm, minimum height=1.0cm, align=center, text width=3.7cm] (box1)  {
\begin{minipage}{3.7cm}
    \centering
    \large Raw Observations ($\bo$)
    \includegraphics[width=0.9cm]{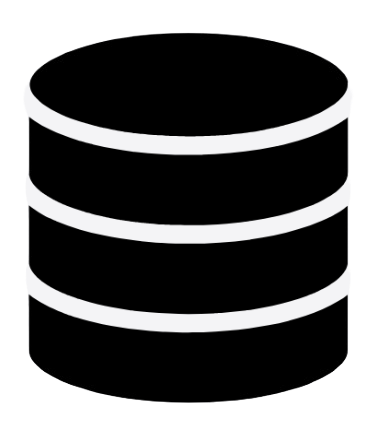}
  \end{minipage}
};
    
    \node[ draw, fill=mylightblue, rounded corners, minimum width=3cm, minimum height=1.0cm, align=center, text width=3cm, right=of box1] (box2)  {
\begin{minipage}{3cm}
    \centering
  \large  PPO Policy $\Pi^\mathrm{PPO}_{\theta}$ (Lipschitz)\\
    \includegraphics[width=1.3cm]{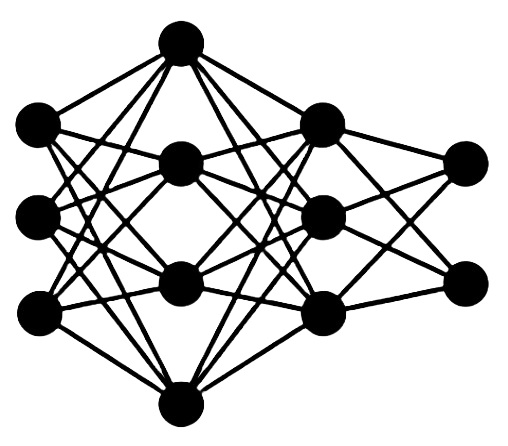}
  \end{minipage}
};

       \node[ draw, fill=blue!30, rounded corners, minimum width=3.7cm, minimum height=1.0cm, align=center, text width=3.7cm, right=of box2] (box3)  {
\begin{minipage}{3.7cm}
    \centering
  \large  Joint Desired Poses $\boldsymbol{q}_d^\mathrm{ppo}$\\
    \includegraphics[width=1.4cm]{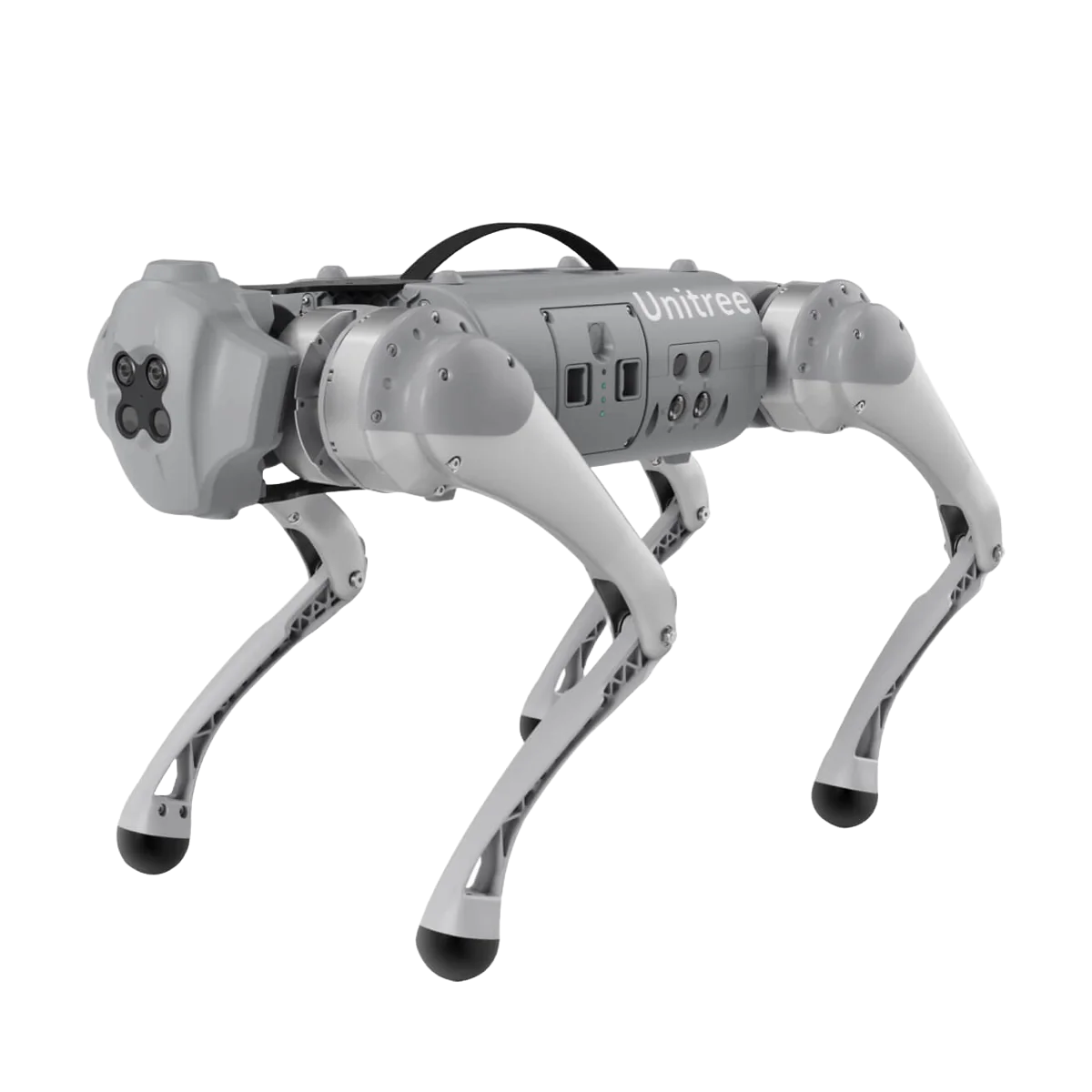}
  \end{minipage}
};
    \node[block4]  (box4) [above=of box3,yshift=3mm]{\large PPO \& \\\large Contraction\\\large Rewards};

  
    \node[block11] (box11)[below=1.2cm of box3]{\large Loss  $\mathcal{L}_{\text{total}}(\theta,\phi)$\\ \large Eqn. \eqref{eqn:loss} };

       \node[ draw, fill=mylightcyan, rounded corners, minimum width=3.1cm, minimum height=1.0cm, align=center, text width=3.1cm, below=1.2of box11] (box8){
     \begin{minipage}{3.0cm}
    \centering
 \large   Contraction loss \\
    $\mathcal{L}_{\text{contr}}(\phi)$ \eqref{eqn:loss_contr}\\
    $\epsilon_\alpha>0$ (Lemma~\ref{lemma:epsilon})
  \end{minipage}};

   \node[ draw, fill=mylightgreen, rounded corners, minimum width=3.5cm, minimum height=1.0cm, align=center, text width=3.5cm, left=of box8] (box7)  {
\begin{minipage}{3.5cm}
    \centering
  \large  Contraction metric \\$M_{\phi}=\Theta^\top\Theta$ (Lipschitz)\\ $V_\phi=\boldsymbol{e}^\top M_\phi \boldsymbol{e}$
    \includegraphics[width=1.5cm]{tikz/nn_no_bg.png}
  \end{minipage}
};
  \node[ draw, fill=mylightorange, rounded corners, minimum width=3.2cm, minimum height=1.0cm, align=center, text width=3.2cm,xshift=-3mm, left=of box7] (box6)  {
\begin{minipage}{3.2cm}
    \centering
   \large  Raw and privileged\\Observations ($\bx$)
    \includegraphics[width=1.4cm]{tikz/go1.png}
  \end{minipage}
};

  \node[ draw, fill=mylightgreen, rounded corners, minimum width=2.0cm, minimum height=1.0cm, align=center, text width=2.0cm,xshift=-3mm, right=1.7cm of box3] (box12)  {\large PD \\ controller $\Pi^\mathrm{PD}_\theta$\\ (200 Hz)
};
  
    \draw[->,line width=0.5mm] (box1) -- (box2);
    \draw[->,line width=0.5mm] (box2) -- (box3);
     \draw[->,line width=0.5mm] (box3) -- (box12);
    \draw[<-,line width=0.5mm] (box2.north) |- (box4.west);

    \draw[->,line width=0.5mm] (box6) -- (box7);
    \draw[->,line width=0.5mm] (box7) -- (box8);

    \draw [->,line width=0.5mm] (box11.west) to[out=180, in=270] node[above, draw=none] {Backprop} (box2.south);

\draw [->,line width=0.5mm] (box11.west) to[out=180, in=90] node[below, draw=none] {Backprop} (box7.north);
    \draw[->, line width=0.5mm] (box3.south) -- (box11.north);
     \draw[->, line width=0.5mm] (box8.north) -- (box11.south) ;
     \draw [<-,line width=0.5mm] (box4.east) to[out=-40, in=40] node[above, draw=none] {} (box11.east);
     
\path let \p1 = (box2), \p2 = (box3) in
  coordinate (mid23) at ($(\p1)!.45!(\p2)$);
\path let \p1 = (box6), \p2 = (box7) in
  coordinate (mid67) at ($(\p1)!.45!(\p2)$);

\draw[->, line width=0.5mm]
  (mid23) to[out=-100, in=100, looseness=1.3] (mid67);

     \coordinate (targetAbove) at ([xshift=0.5cm, yshift=2.5cm]box12.east);
\coordinate (targetBelow) at ([xshift=0.5cm, yshift=-3.5cm]box12.east);

\draw[->, line width=0.5mm] (box12.east) to[out=0, in=180] (targetAbove);
\draw[->, line width=0.5mm] (box12.east) to[out=0, in=180] (targetBelow);

\begin{scope}[on background layer]
  \path let \p1=(box1.north west), \p2=(box12.south east) in
    coordinate (outerNW) at ([xshift=-2mm,yshift=5mm]\p1)
    coordinate (outerSE) at ([xshift=2mm,yshift=-6mm]\p2);
  \fill[blue!20, opacity=0.4, even odd rule, rounded corners=8pt]
    (outerNW) rectangle (outerSE)
    (box1.south west) rectangle (box1.north east)
    (box2.south west) rectangle (box2.north east)
    (box3.south west) rectangle (box3.north east)
    (box12.south west) rectangle (box12.north east);
  \draw[rounded corners=8pt, thick, blue!50]
    (outerNW) rectangle (outerSE);
\end{scope}

\begin{scope}[on background layer]

  \path let \p1=(box6.north west), \p2=(box8.south east) in
    coordinate (outerNW2) at ([xshift=-2mm,yshift=6mm]\p1)
    coordinate (outerSE2) at ([xshift=2mm,yshift=-8mm]\p2);

  \fill[green!20, opacity=0.4, even odd rule, rounded corners=4pt]
    (outerNW2) rectangle (outerSE2)
    (box6.south west) rectangle (box6.north east)
    (box7.south west) rectangle (box7.north east)
    (box8.south west) rectangle (box8.north east);

  \draw[rounded corners=8pt, thick, green!50]
    (outerNW2) rectangle (outerSE2);

\end{scope}

  \end{tikzpicture}%
        \centering
        \includegraphics[width=0.22\linewidth]{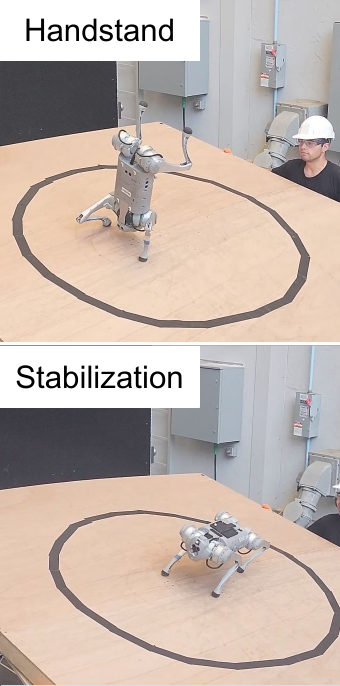}
    \end{minipage}
    \begin{minipage}[t]{0.3\textwidth}
    \vspace{-5.6cm}
      \scalebox{0.88}{
    \begin{minipage}{\linewidth}
    \begin{algorithm}[H]
\small
\caption{\footnotesize \ours}
\label{alg:contractionppo-compact}
\begin{algorithmic}[1]
\WHILE{not converged}
  \STATE Collect rollouts $(\bo_t,\bx_t,r_t)$ with $\bu_t=\Pi^\mathrm{PD}_\theta(\bo_t)$ and desired poses $q^{\mathrm{PPO}}_d$ \eqref{eq:ppo_decompose} from $\Pi^\mathrm{PPO}_\theta(\bo_t)$
  \STATE Compute $\mathcal{L}_{\text{PPO}}(\theta)$ (from \cite{schulman2017proximal_ppo})
  \FORALL{samples $(\bo,\bx)$}
    \STATE Form $\bx_d$ and set $\boldsymbol{e}=\bx-\bx_d$
    \STATE Compute $V_\phi(\bx)$ via \eqref{eq:Vphi} 
    \STATE Compute $A_{\mathrm{cl}}(\bx)$ from \eqref{eq:Acl} and $\dot M_\phi$ from \eqref{eqn:M} 
    \STATE Evaluate $\dot V_\phi$ and $\mathcal{L}_{\text{contr}}$ from \eqref{eqn:loss_contr} and $\mathcal{L}_{\text{PD}}$
  \ENDFOR
  \STATE Minimize $\mathcal{L}_{\text{total}}$ from \eqref{eqn:loss} 
\ENDWHILE
\STATE Deploy $\Pi^\mathrm{PD}_\theta$ on the robot (learned $M_\phi$ is used in the simulation only)
\end{algorithmic}
\label{alg:contractionppo}
\end{algorithm}
\end{minipage}
}
    \end{minipage}
    \caption{\footnotesize Architecture of ContractionPPO (left) and its pseudocode (right). The PPO policy $\Pi^\mathrm{PPO}_\theta$ processes raw observations $\bo$ and outputs desired joint poses $\boldsymbol{q}_d^\mathrm{ppo}$, which are executed by a low-level PD controller~\eqref{eq:pid_joint}. In parallel, the contraction metric $M_\phi$ receives privileged and raw observations $\bx=[\boldsymbol{q}^\mathrm{ppo}(t),\;\dot{\boldsymbol{q}}^\mathrm{ppo}(t),\;\boldsymbol{p}(t),\;\boldsymbol{v}(t)]$ and outputs a positive definite metric $M_\phi = \Theta^\top \Theta$. The contraction loss is evaluated using the Lyapunov condition $\dot{V} + \alpha V \leq -\epsilon_\alpha$, where $\alpha$ (satisfying \eqref{eqn:alpha_lower_bound}) specifies the desired contraction rate and $\epsilon_\alpha$ quantifies the approximation margin between the learned value function $V_\phi$ and the true contraction Lyapunov function $V$. A larger $\alpha$ leads to faster convergence guarantees but also requires a larger $\epsilon_\alpha$. This joint training setup ensures that the policy not only maximizes task reward but also satisfies certifiable incremental stability guarantees during locomotion.
}
    \label{fig:approach}
\end{figure*}
{

\section{Related Work\label{sec:related_work}}
Certified locomotion control in high-dimensional systems remains a challenge due to the need to balance performance, generalization, and formal guarantees. We broadly categorize related work into: (a) classical and hybrid learning-based control methods, (b) learning-based locomotion, (c) deep RL for legged locomotion and (d) sim-to-real transfer.

\subsubsection{Classical and Hybrid Learning-based Control}
Traditional controllers \cite{akbas2012zero_zmp_1} offer robustness through model-based design but struggle with generalization to unstructured settings. Hybrid approaches embed control-theoretic tools like CBFs into learning \cite{choi2020reinforcement_cbf1}, but often require handcrafted functions and online optimization, limiting scalability for high-DOF robots. In contrast, our approach provides a trajectory-level robustness guarantee i.e., nearby trajectories provably contract, and leads to bounded tracking error.



\textit{Learning-based locomotion without guarantees.}
Deep RL has enabled agile locomotion beyond the reach of hand-engineered controllers, with PPO-based policies demonstrating impressive performance on quadrupeds in simulation and hardware \cite{katz2019mini_ref1,hutter2016anymal_ref2,jacoff2023taking_ref3,schulman2017proximal_ppo,radosavovic2024real_malik,tsounis2020deepgait,bellegarda2024robust_1}. However, standard RL lacks formal guarantees; i.e., stability and robustness are not certified while performance may degrade under disturbances, sensor noise, or mismatched dynamics \cite{kumar2022adapting_rma_limit1,choi2023learning_limit2,luo2024moral_limit3,fu2023deep,yoon2024learning}. On the other hand, works that integrate RL based approaches with control theoretic concepts often assume full state observability \cite{han2020actor_full_state_1, pavlichenko2022real_full_state_2}. This assumption is not practical especially for high dimensional systems. Methods that leverage privileged information for adaptation (e.g., system-identification heads or latent context) improve sim-to-real transfer but still stop short of certifying closed-loop stability.

\textit{Deep RL for Legged Locomotion}:
Several works have shown that deep RL can produce agile and dynamic behaviors in quadruped robots. \cite{tan2018sim2real_quadruped} used PPO with heavy domain randomization to train a robot to trot and bound in the real world. \cite{hwangbo2019learning_anymal} learned torque policies that enable a robot to walk fast and recover from disturbances, outperforming hand-designed controllers. \cite{lee2020learning_terrain} extended these results to rough terrain, training policies that walk over stairs and gaps using privileged information and robust training. \cite{rudin2022learning} showed that training with thousands of parallel environments in simulation can produce walking behaviors in minutes. More recently, \cite{miki2022perceptive} combined vision with RL to enable robots to walk over visually complex terrain.
While these methods show strong performance, they do not provide any formal guarantee of stability or robustness, especially when deployed on hardware. They rely heavily on empirical robustness from randomization, and often struggle with out-of-distribution inputs. In contrast, our method augments PPO with contraction-based stability to certify stability margins directly in training, improving safety and reliability without sacrificing performance.

\textit{Sim-to-Real Transfer}:
Another set of works tackles sim-to-real transfer by making policies more robust to modeling errors and disturbances. Rapid Motor Adaptation (RMA) learns a latent variable from proprioception that allows online adaptation across different terrains and payloads \cite{kumar2022adapting_rma_limit1}. \cite{miki2022perceptive} combined learned perception and terrain-aware control to handle visual uncertainty and perturbations. \cite{rudin2022learning} further improved sim-to-real success using large batch sizes, short episodes, and strong domain randomization during training.
These approaches improve robustness through adaptation or data augmentation, but they still lack formal guarantees for safety/stability. Their behavior under unseen conditions or large disturbances remains uncertain. 
Our framework addresses this limitation by embedding a differentiable contraction-based metric into PPO, certifying incremental stability even under disturbances and domain shift. 

Unlike prior deep RL methods that rely on empirical robustness or online constraint enforcement, our approach certifies incremental exponential stability of the closed-loop dynamics using contraction theory. This provides a trajectory-level robustness guarantee and integrates directly into PPO while retaining an observation-based policy at deployment.


}
\section{Preliminaries and problem statement\label{sec:prelim_and_problem}}
In this section, we discuss the preliminaries and the problem that we address in this paper.

\subsection{System Dynamics}
In this paper, we consider a quadruped robot, described by general nonlinear control-affine dynamics
\begin{align}
    \dot{\bx} = f(\bx) + B(\bx)\bu+\boldsymbol{d}(\bx,t),\quad\bx(0)=\bx_0,
    \label{eq:dynamics}
\end{align}
where $\bx \in \mathbb{R}^n$, $\bu \in \mathbb{R}^m$,  $f(\bx)$ models the drift dynamics, $B(\bx)$ is the input matrix, both assumed $C^1$ smooth and $\|\boldsymbol{d}(\bx,t)\|\leq\bar{d}$. 
The agent receives observation $\bo = h(\bx) \in \mathbb{R}^p$, where $h:\mathbb{R}^n\to\mathbb{R}^p$ is a smooth observation map (e.g., proprioception, contacts, IMU).
A parameterized stochastic policy $\Pi^\mathrm{PD}_\theta: \mathbb{R}^p \rightarrow \mathbb{R}^m$ maps observation $\bo$ to action $\bu = \Pi^\mathrm{PD}_\theta(\bo)$, yielding the closed-loop vector field:
\begin{align}
    \dot{\bx} = f_{\mathrm{cl}}(\bx) := f(\bx) + B(\bx) \Pi^\mathrm{PD}_\theta\big(h(\bx)\big).
    \label{eq:closedloop}
\end{align}
assuming no disturbances.
\subsection{Contraction Analysis of Nonlinear Closed-Loop Systems\label{subsec:contr_theory}}
Contraction theory \cite{tsukamoto2021contraction_tutorial,lohmiller1998contraction} reformulates Lyapunov stability conditions by employing a quadratic function of differential states, defined through a Riemannian contraction metric with a uniformly positive definite matrix. This framework establishes necessary and sufficient conditions for the incremental exponential convergence of trajectories in nonlinear dynamical systems. Let $M(\bx):\mathbb{R}^n \to \mathbb{S}_{++}^n$ be a $C^1$ symmetric positive definite metric field, with uniform bounds $ 0 < m_{\min} I \preceq M(\bx) \preceq m_{\max} I < \infty$
where $m_{\max}>m_{\min}>0$. Consider the quadratic differential Lyapunov function $ V(\delta\bx, \bx) = \delta\bx^\top M(\bx) \delta\bx$
where $\delta\bx$ is an infinitesimal displacement between two trajectories. Along solutions of~\eqref{eq:closedloop}, $\delta\bx$ evolves via the variational (differential) dynamics $ \delta\dot{\bx} = A_{\mathrm{cl}}(\bx) \delta\bx,$
with
\begin{align}
\small
    A_{\mathrm{cl}}(\bx) = &\frac{\partial f}{\partial x}(\bx) + \sum_{i=1}^m \left(\frac{\partial B_i}{\partial x}(\bx)\right)\Pi^\mathrm{PD}_{\theta,i}\big(h(\bx)\big) \nonumber \\
    &+ B(\bx) J_\pi\big(h(\bx)\big) J_h(\bx),
    \label{eq:Acl}
\end{align}
where $J_\pi(\bo) = \frac{\partial \Pi^\mathrm{PD}_\theta}{\partial o}(\bo)$ and $J_h(\bx) = \frac{\partial h}{\partial x}(\bx)$. The time derivative of the metric along system trajectories \eqref{eq:closedloop} is
\begin{align}
    \dot{M}(\bx) 
=\sum_{k=1}^n \frac{\partial M_\phi}{\partial x_k}(\bx) \big(f_{\mathrm{cl}}(\bx)\big)_k.
\label{eqn:M}
\end{align}
The time derivative of $V$ along the coupled dynamics is $    \dot{V}(\delta\bx, \bx) = \delta\bx^\top \Big( A_{\mathrm{cl}}^\top M + M A_{\mathrm{cl}} + \dot{M} \Big) \delta\bx$.
    \label{eq:Vdot}
A uniform contraction rate $\alpha > 0$ is certified if
\begin{align}
    & A_{\mathrm{cl}}^\top M + M A_{\mathrm{cl}} + \dot{M} +\alpha M  \preceq 0,\;\;\forall \bx\in\mathbb{R}^n.
    \label{eq:contraction-ineq}
\end{align}
This matrix inequality is both necessary and sufficient for IES of the closed-loop system in the metric $M(\bx)$~\cite{lohmiller1998contraction,tsukamoto2021contraction_tutorial}.
\subsection{Deep Reinforcement Learning}
In a standard reinforcement learning (RL) framework, the agent interacts with the robot environment in discrete time. At each step, the agent receives observation $\bo = h(\bx)$ and selects an action $\bu = \Pi^\mathrm{PPO}_\theta(\bo)$ according to a parameterized stochastic policy $\Pi^\mathrm{PPO}_\theta$. The system then transitions according to the dynamics \eqref{eq:dynamics}, and the agent receives a reward $r_t$, designed to encourage agile, stable locomotion.
The objective is to maximize the expected cumulative reward over episodes i.e., $
    J(\theta) = \mathbb{E}\left[ \sum_{t=0}^\infty \gamma^t r_t \right]$
where $\gamma \in (0,1)$ is a discount factor. To optimize this objective, we use Proximal Policy Optimization (PPO) \cite{schulman2017proximal_ppo}, a popular on-policy algorithm that updates $\theta$ by maximizing a clipped surrogate objective while regularizing policy entropy and the value prediction error. PPO is known for stable and efficient training in continuous control tasks, and forms the backbone of our locomotion policy learning.
The policy $\Pi^\mathrm{PPO}_\theta$ is implemented as a neural network that maps proprioceptive observations to low-level joint commands.

\subsection{Problem Statement}
We consider not only performance but also formal guarantees of stability/safety for any learned controller for locomotion. We formalize this as follows:
\begin{problem}
Given a quadruped with dynamics~\eqref{eq:dynamics} and observation model $\bo = h(\bx)$, learn a feedback policy $\Pi^\mathrm{PPO}_\theta$ and contraction metric $M_\phi(\bx)$ such that: (i) the expected PPO reward is maximized, and (ii) the closed-loop contraction condition~\eqref{eq:contraction-ineq} holds with rate $\alpha > 0$.
\end{problem}
The main challenge is to provide {provable guarantees} that the learned policy not only achieves task reward, but is certified to be exponentially stable, even under high-dimensional nonlinear feedback typical in deep RL for robotics.

\section{Proposed Approach\label{sec:proposed_approach}}
In this section, we divide the proposed approach into two steps. First, we discuss the algorithmic details of our proposed approach ContractionPPO. Second, we discuss about the theoretical guarantees that \ours provides.
\subsection{\ours Algorithm}
Our ContractionPPO method augments PPO with a learned contraction metric for formal closed-loop certification in agile legged locomotion. The core idea is to train two parallel neural networks i.e., a standard PPO policy $\Pi^\mathrm{PPO}_\theta$ and a contraction metric $M_\phi$ (parameterized by $\phi$) both of which are Lipschitz-constrained. These are jointly optimized via a composite loss that combines task performance with a Lyapunov based contraction regularizer, leveraging both raw and privileged information during training.

{
Let $\boldsymbol{y}=\boldsymbol{q}$ denotes the observable joint configuration, where each joint angle is represented using a quaternion-based parameterization and concatenated into a single vector. In joint space with $(\boldsymbol{q},\dot{\boldsymbol{q}})$, the commanded low-level torques are given by
\begin{equation}
\Pi^\mathrm{PD}_\theta = K_p\big(\boldsymbol{q}^{\mathrm{ppo}}_d - \boldsymbol{q}\big)\;+\;K_d\big(\dot{\boldsymbol{q}}^{\mathrm{ppo}}_d - \dot{\boldsymbol{q}}\big)
\label{eq:pid_joint}
\end{equation}
with gains $K_p,K_d$.
To clarify the role of the learned policy $\Pi^\mathrm{PPO}_\theta(\bo)$, we decompose
\begin{equation}
\boldsymbol{q}^{\mathrm{ppo}}_d = \boldsymbol{q}_{\mathrm{d}} + \Delta \boldsymbol{q}_\theta(\bo),
\label{eq:ppo_decompose}
\end{equation}
where $\Delta\boldsymbol{q}_\theta(\bo)$ is a learned nonlinear feedback term produced by the policy $\Pi^\mathrm{PPO}_\theta(\bo)$ from sensor observations $\bo$. One of the innovation lies in learning $\Delta \boldsymbol{q}_\theta$ as an additional nonlinear feedback term, while the contraction metric certifies incremental exponential stability of the resulting closed loop.
}
The desired trajectory $\bx_d(t)$ is given by $\bx_d(t)=[\boldsymbol{q}^\mathrm{ppo}_d(t),\;\dot{\boldsymbol{q}}^\mathrm{ppo}_d(t),\;\boldsymbol{p}_\mathrm{d}(t),\;\boldsymbol{v}_\mathrm{d}(t)]$ where $\boldsymbol{p}_\mathrm{d}(t)$ and $\boldsymbol{v}_\mathrm{d}(t)$ are desired position and velocity of center of mass (COM) of the robot.

The overall architecture is illustrated in Fig.~\ref{fig:approach}. The PPO policy network receives raw sensor observations $\bo$ as input and outputs a vector of target joint poses (denoted by $\boldsymbol{q}^\mathrm{ppo}_d$). Consequently, the PD controller~\eqref{eq:pid_joint} is leveraged to actuate the robot in a feedback manner. Parallel to the policy, the contraction metric is fed with a richer state  vector $\bx=[\boldsymbol{q}^\mathrm{ppo}(t),\;\dot{\boldsymbol{q}}^\mathrm{ppo}(t),\;\boldsymbol{p}(t),\;\boldsymbol{v}(t)]$ that concatenates the raw policy outputs $\boldsymbol{q}^\mathrm{ppo}$ and additional privileged information such as $\boldsymbol{p}(t)$ and $\boldsymbol{v}(t)$ state variables that are available only in simulation. This privileged information is only used to guide training, not deployed at test time.

The contraction metric $M_\phi(\bx,t)$ (parameterized via a matrix-valued MLP (multi-layer perceptron) function $\Theta(\bx,t)$ with parameters $\phi$) outputs a vectorized representation of a lower-triangular matrix $\Theta(\bx,t)$, which is reshaped and multiplied by its transpose to produce a symmetric positive definite matrix $M_\phi(\bx,t) = \Theta(\bx)^\top \Theta(\bx)$. This parameterization ensures $M_\phi(\bx,t) \succ 0$ by construction, which is critical for certifying contraction. 
Let \(\boldsymbol{e}(t) := \bx(t)-\bx_d(t)\) and define the smooth parameterized path $\bx(\mu,t)$ as $\bx(\mu=0,t):=\bx(t)$ and $\bx(\mu=1,t)=\bx_d$ (one particular choice is $\bx(\mu,t)=\bx_d+\mu\boldsymbol{e}$).
Define the geodesic distance~\cite{tsukamoto2021contraction_tutorial} as 
\begin{align}
V_\phi(\bx,t) 
&:=\boldsymbol{e}^\top M_\phi(\bx)\boldsymbol{e} 
\label{eq:Vphi}
\end{align}
For the system to be contracting with rate $\alpha > 0$, the inequality $\dot{V}_\phi(\bx,t) + \alpha V_\phi(\bx,t) \leq -\epsilon_\alpha$ (Theorem 2.3 in \cite{tsukamoto2021contraction_tutorial}) must hold 
for some margin $\epsilon_\alpha > 0$. Allowing $M_\phi$ to vary with state enables stronger, more adaptive contraction behavior than a fixed metric (see Table \ref{table:comparison}). Note that as $V_\phi\leq V_\ell:= \int_{0}^{1}\Big(\tfrac{\partial \bx}{\partial \mu}\Big)^\top M_\phi\big(\bx(\mu,t),t\big)\,\Big(\tfrac{\partial \bx}{\partial \mu}\Big)\,d\mu/m_{\min}$ \cite{tsukamoto2021contraction_tutorial}, ensuring contraction of the metric $V_\ell$ also implies that $\|\boldsymbol{e}\|$ decays exponentially i.e., all trajectories converge towards $\bx_d$ exponentially.
 To enforce the contraction property during learning, we define a hinge loss that penalizes any violation of the target inequality:
\begin{align}
    \mathcal{L}_{\text{contr}} = \mathrm{ReLU}\left(\frac{\dot{V}_\phi(\bx,t) + \alpha V_\phi(\bx,t)}{V_\phi(\bx,t)} + \underline{\epsilon_\alpha}\right).
    \label{eqn:loss_contr}
\end{align}
where the parameter $\underline{\epsilon_\alpha}={\epsilon_\alpha}/V_\phi(\bx,t)$ and $\epsilon_\alpha$ reflects the approximation gap between the learned $V_\phi$ (approximate) and the actual Lyapunov function $V$ (unknown). Its lower bound is defined in Lemma~\ref{lemma:epsilon}, while the minimum required value of $\alpha$ is given in \eqref{eqn:alpha_lower_bound}. As $\alpha$ increases, the value of $\epsilon_\alpha$ must also increase accordingly. This creates a trade off i.e., although a larger $\alpha$ implies a faster rate of convergence, it simultaneously complicates the learning process. This is because the contraction condition forces the learned value function to satisfy a stricter upper bound, making optimization more challenging as $\alpha$ increases.
\begin{figure}[t]
    \centering
    \definecolor{lightblue3}{HTML}{CFE2F3}

    \begin{tikzpicture}[>=Latex,scale=0.45]

        \def\npoints{1500}   
        \def\radx{1.6}      
        \def\rady{2.3}      

        \foreach \i in {1,...,\npoints} {
            \pgfmathsetmacro{\t}{rnd*360} 
            \pgfmathsetmacro{\r}{sqrt(rnd)} 
            \pgfmathsetmacro{\x}{\r*\radx*cos(\t) + 0.3*sin(\t*2)}
            \pgfmathsetmacro{\y}{\r*\rady*sin(\t) + 0.2*cos(\t*3)}
            \fill[blue,opacity=0.8] (\x,\y) circle (0.035cm);
        }

        \draw[very thick,-{Latex[length=5mm,width=3.5mm]}] (2.6,0) -- (4.6,0);

        \begin{scope}[xshift=6.8cm]
            \path[fill=blue,opacity=0.9,draw=blue!70!black,thick]
                plot [domain=0:360,smooth,variable=\t]
                ({\radx*cos(\t) + 0.3*sin(2*\t)},
                 {\rady*sin(\t) + 0.2*cos(3*\t)});
        \end{scope}

        \node at (3.3,-2.9) {$\dot{V}_\phi+\alpha V_\phi\leq-\epsilon_\alpha\;\;\;\implies\;\;\; \dot{V}+\alpha V\leq 0$};
        \node at (0.1,2.9) {Simulation};
        \node at (6.7,2.9) {Real};
    \end{tikzpicture}
    \caption{\footnotesize Enforcing a simulation based margin $\epsilon_\alpha$ on $\dot V_\phi+\alpha V_\phi$ over sampled points yields a uniform certificate $\dot V+\alpha V\le 0$ at deployment (Theorem \ref{thm:PA-lip}) over the compact state space $\mathcal{K}$, which represents the region of the state space explored during training.}
    \label{fig:vsimequalsvreal}
\end{figure}
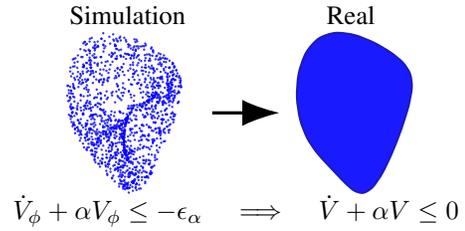
\begin{figure*}[ht]
    \centering
    \includegraphics[width=\linewidth]{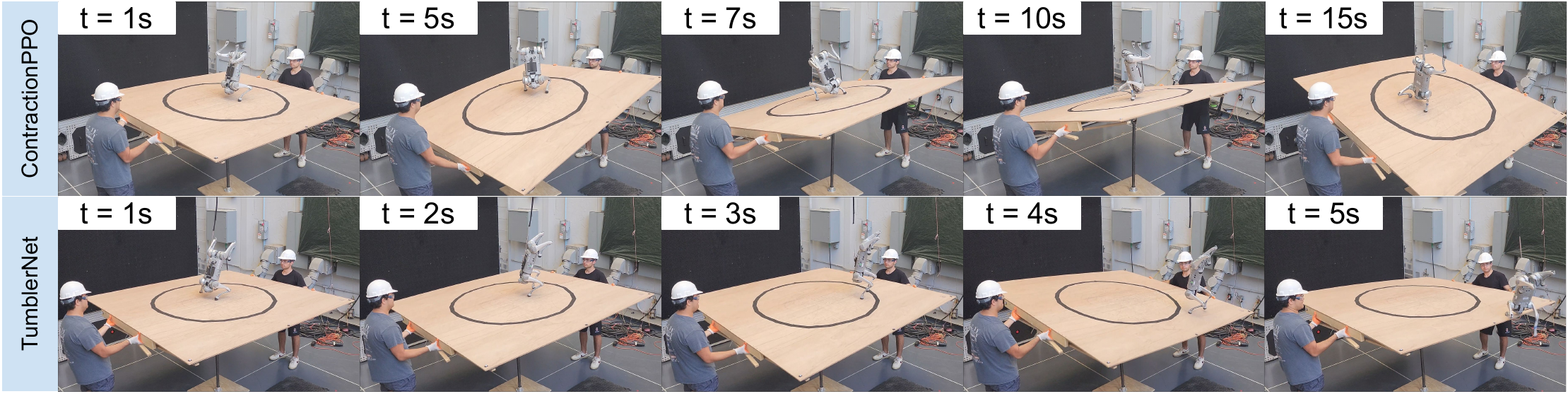}
    \caption{\footnotesize Comparison of handstand using ours \ours (top) and TumblerNet \cite{xiao2025learning} (bottom). 
        TumblerNet struggles to maintain static balance, while \ours enforces 
        contraction based stability, resulting in more consistent and robust gait control.
    }
    \label{fig:ppo_vs_contraction_handstand_ground}
    \vspace{-0.25cm}
\end{figure*}
The full training loss aggregates the PPO objective (including value and entropy terms), the contraction loss, and additional soft penalties to ensure $M_\phi(\bx,t)$ remains well conditioned:
\begin{align}
    \mathcal{L}_{\text{total}}(\theta,\phi) = \mathcal{L}_{\text{PPO}}(\theta) + w_{\text{contr}} \mathcal{L}_{\text{contr}}(\phi) +
    w_{\text{PD}} \mathcal{L}_{\text{PD}}(\phi).
    \label{eqn:loss}
\end{align}
where $\mathcal{L}_{\text{PPO}}$ is given in \cite{schulman2017proximal_ppo}, $\mathcal{L}_{\text{PD}}$ penalizes $M_\phi(\bx,t)$ if its smallest eigenvalue falls below a set threshold or if its largest eigenvalue exceeds an upper bound, thus enforcing the metric bounds required for contraction theory. Mathematically, $\mathcal{L}_{\text{PD}}(\phi)=
\mathrm{ReLU}(m_{\min} - \lambda_{\min}(M_\phi(\bx))) \allowbreak + \allowbreak
\mathrm{ReLU}(\lambda_{\max}(M_\phi(\bx)) - m_{\max})$
. Note that if $\epsilon_\alpha$ is chosen satisfying the conditions in Lemma~\ref{lemma:epsilon}, then minimizing \eqref{eqn:loss} would guarantee that the contraction condition is satisfied for the compact state and action domain space. The pseudocode for \ours is given in Algorithm \ref{alg:contractionppo}.

\subsection{Theoretical Guarantees\label{subsec:guarantees}}
{
At each training step, \ours rolls out the closed-loop \eqref{eq:closedloop} with inputs $ \bu=\Pi^\mathrm{PD}_\theta(\bo) $ and observations $ \bo=h(\bx) $, computes the variational matrix $A_{\mathrm{cl}}(\bx)$ in \eqref{eq:Acl} (computed via $\texttt{autograd}$), and evaluates the differential Lyapunov function \eqref{eq:Vphi} and its derivative $\dot{V}_\phi$ (with $\dot M_\phi$ computed via $\texttt{autograd}$). 
To make the certified margin explicit, the algorithm enforces Lipschitz constraints on the networks ($\|J_\pi\|\le L_\pi$, $\|\nabla M_\phi\|\le L_M$) and uses dynamics, which feed Theorem~\ref{thm:PA-lip} and Corollary~\ref{cor:margin} to compute an explicit upper bound for the contraction rate $\alpha$. The tolerance $\epsilon_\alpha$ in the loss \eqref{eqn:loss_contr} quantifies approximation between the learned $V_\phi$ and the true $V$.  Lemma~\ref{lemma:epsilon} provides a state proportional lower bound for $\epsilon_\alpha$ and clarifies its trade off with $\alpha$. Finally, with disturbances $d$ in \eqref{eq:dynamics}, Theorem~\ref{thm:PA-robust} lifts the certificate to the input to state incremental bound \eqref{eq:PA-iss}.
}
Defining the contraction residual $r(\bx,\boldsymbol{e})=\boldsymbol{e}^\top\mathcal{R}(\bx)\boldsymbol{e}
$
where $\mathcal{R}(\bx)=\Big(A_{\text{cl}}^\top M_\phi+M_\phi A_{\text{cl}}+\dot M_\phi+\alpha M_\phi\Big)$. Define $\hat{\boldsymbol{e}}$ and normalized residual operator $\hat{\mathcal{R}}(\bx)$ as 
\begin{align}
  \hat{\boldsymbol{e}}=\frac{M_\phi(\bx)^{\frac{-1}{2}}\boldsymbol{e}}{\|M_\phi(\bx)^{\frac{-1}{2}}\boldsymbol{e}\|},\;\;\; \hat{\mathcal{R}}(\bx)=M_\phi(\bx)^{\frac{-1}{2}} \mathcal{R}(\bx)M_\phi(\bx)^{\frac{-1}{2}} 
  \label{eq:PA-Lcontr}
\end{align}
The Rayleigh quotient identity gives $\frac{r(\bx,\boldsymbol{e})}{V_\phi(\bx)}=\frac{\boldsymbol{e}^\top \mathcal{R}(\bx) \boldsymbol{e}}{\boldsymbol{e}^\top M_\phi(\bx) \boldsymbol{e}}=\hat{\boldsymbol{e}}^\top\hat{\mathcal{R}}(\bx)\hat{\boldsymbol{e}}$. Consequently,  $\underset{\boldsymbol{e}\neq 0}{\sup} \frac{r(\bx,\boldsymbol{e})}{V_\phi(\bx,\boldsymbol{e})}=\lambda_{\max}(\hat{\mathcal{R}}(\bx))$

\begin{assumption}[Lipschitz constraints and bounds]\label{ass:lipschitz}
There exist finite constants $L_\pi, L_M, \bar f, \bar B, \bar{u}, \bar J_h$ such that
\begin{align}
&\norm{J_\pi(\bo)}_2 \le L_\pi,\quad 
\norm{\nabla M_\phi(\bx)}_F \le L_M,\quad \norm{f(\bx)}_2 \le \bar f,\nonumber\\
&
\norm{B(\bx)}_2 \le \bar B,\quad
\norm{\Pi^\mathrm{PD}_\theta(\bo)}_2 \le \bar{u},\quad \norm{J_h(\bx)}_2 \le \bar J_h,\nonumber
\end{align}
\end{assumption}

{
To translate the architectural constraints of \ours into verifiable stability guarantees, in the following theorem, we derive an explicit lower bound for the contraction rate $\alpha$ in terms of the Lipschitz constants of the policy and metric networks, combined with bounded envelopes on robot dynamics and the observation Jacobian. This analysis also accounts for simulation-to-reality approximation errors, ensuring that the certified contraction property remains valid when transitioning from privileged training environments to real-world deployment. Define the constants
\begin{align}
&C_f=\sup_{\bx\in\mathcal{K}}\Big\lVert \sym\big(M_\phi^{\frac{-1}{2}}(\tfrac{\partial f}{\partial x}^\top M_\phi+M_\phi\tfrac{\partial f}{\partial x})M_\phi^{\frac{-1}{2}}\big)\Big\rVert_2,\nonumber\\
&C_{B\partial B}=\sup_{\bx\in\mathcal{K}}\Big\lVert \sym\Big(M_\phi^{\frac{-1}{2}}\big(\sum_{i=1}^m \Pi^\mathrm{PD}_{\theta,i}(h) (\tfrac{\partial B_i}{\partial x})^\top M_\phi + \nonumber\\
&\quad \quad M_\phi\sum_{i=1}^m \Pi^\mathrm{PD}_{\theta,i}(h) \tfrac{\partial B_i}{\partial x}\big)M_\phi^{\frac{-1}{2}}\Big)\Big\rVert_2,\label{eq:PA-constants}\\
&C_{BJ}=2\sup_{\bx\in\mathcal{K}}\big\lVert M_\phi^{1/2}B(\bx)M_\phi^{\frac{-1}{2}}\big\rVert_2.\nonumber
\end{align}
}
\begin{theorem}[Sim-to-Real Generalization Guarantee]\label{thm:PA-lip}
Under Assumption \ref{ass:lipschitz}, if the right-hand side of the following bound is negative, then for all $\bx \in \mathcal{K}$, the closed-loop system is incrementally exponentially stable
\begin{align}
&\eigmax\big(\sym(\hat{\mathcal{R}}(\bx))\big)\ \label{eq:PA-RB}\\
&\le\ C_f+C_{B\partial B}+C_{BJ} L_\pi \bar J_h+\frac{L_M}{\sqrt{m_{\min}}}\big(\bar f+\bar B \bar{u}\big)-\alpha.\nonumber
\end{align}
\end{theorem}
\begin{proof}
Recall the normalized residual operator $\mathcal{R}(\bx)$ in \eqref{eq:PA-Lcontr} where $A_{\mathrm{cl}}$ from \eqref{eq:Acl} is given by,
\begin{align}
A_{\mathrm{cl}} = \underbrace{\tfrac{\partial f}{\partial x}}_{f\text{-part}}
\;+\; \underbrace{\sum_{i=1}^m \Big(\tfrac{\partial B_i}{\partial x}\Big)\Pi^\mathrm{PD}_{\theta,i}(h(\bx))}_{\partial B\text{-part}}
\;+\; \underbrace{B J_\pi(h(\bx)) J_h(\bx)}_{BJ\text{-part}},\nonumber
\end{align}
we split
\begin{align}
\sym(\hat{\mathcal{R}}) \!=\! \sym(\mathcal{R}_f)\!+\!\sym(\mathcal{R}_{\partial B})\!+
\!\sym(\mathcal{R}_{BJ})\!+\!\;\mathcal{R}_{\dot M}\!+\!\alpha I,\nonumber
\end{align}
where each symbol denotes the normalization by $M_\phi^{\frac{-1}{2}}(\cdot)M_\phi^{\frac{-1}{2}}$ as in \eqref{eq:PA-Lcontr} applied to the corresponding term, and
$\mathcal{R}_{\dot M}:=M_\phi^{\frac{-1}{2}}\dot M_\phi M_\phi^{\frac{-1}{2}}$.

\noindent\textbf{(i) $f$ and $\partial B$ parts.}
By definition of the constants in the theorem,
\begin{align}
\big\|\sym(\mathcal{R}_f)\big\|_2 \le C_f,
\qquad
\big\|\sym(\mathcal{R}_{\partial B})\big\|_2 \le C_{B\partial B}.
\end{align}
This is tautological once $C_f$ and $C_{B\partial B}$ are defined as the suprema of the corresponding spectral norms over $\bx$.






\begin{figure*}[ht]
    \centering
    \includegraphics[width=\linewidth]{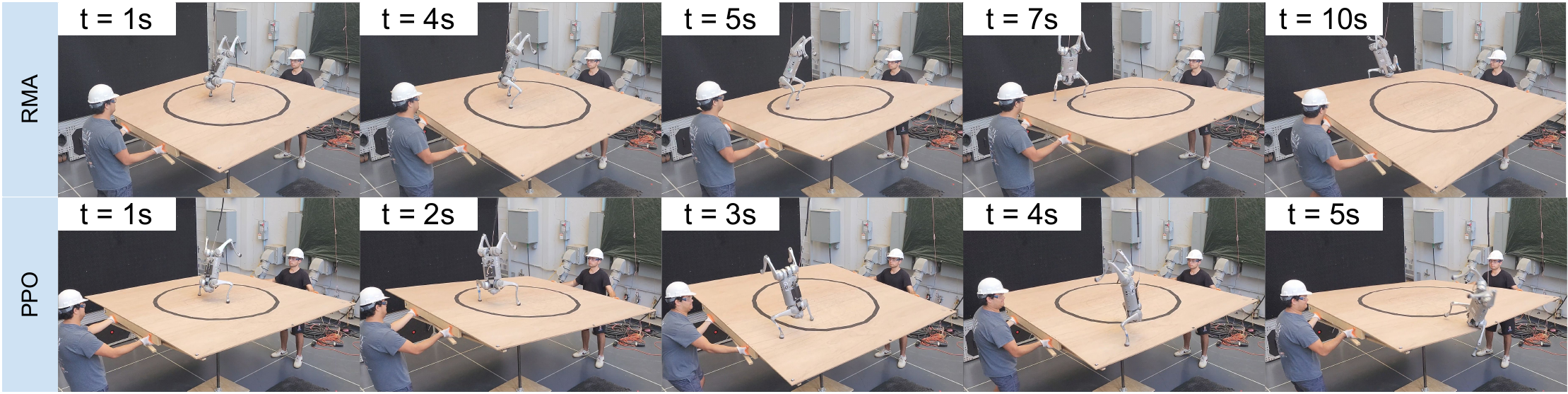}
    \caption{\footnotesize Comparison of handstand control using RMA \cite{kumar2022adapting_rma_limit1} (top) and PPO \cite{schulman2017proximal_ppo} (bottom). 
        The robot using PPO and RMA struggles to maintain balance and falls off the platform.
    }
    \label{fig:ppo_vs_contraction_handstand}
    \vspace{-0.25cm}
\end{figure*}

\noindent\textbf{(ii) Policy geometry ($BJ$) part.}
Let $Y := M_\phi^{1/2} B J_\pi J_h M_\phi^{\frac{-1}{2}}$. Then
\begin{align}
\sym(\mathcal{R}_{BJ})=  &\frac{1}{2}\left(M_\phi^{\frac{-1}{2}}(B J_\pi J_h)^\top M_\phi M_\phi^{\frac{-1}{2}}\right.
\\
&\left.+M_\phi^{\frac{-1}{2}} M_\phi (B J_\pi J_h) M_\phi^{\frac{-1}{2}}\right)=\sym(Y)\nonumber
\end{align}
Hence, by the triangle inequality and invariance of the spectral norm under transpose,
\begin{align}
\big\|\sym(\mathcal{R}_{BJ})\big\|_2
\!=\! \big\|\sym(Y)\big\|_2
\!\le\! \frac{1}{2}\big(\|Y^\top\|_2 \!+\! \|Y\|_2\big)
\!=\! \|Y\|_2.\nonumber
\end{align}
Submultiplicativity for the spectral norm gives
\begin{align}
\|Y\|_2
\le \big\|M_\phi^{1/2} B M_\phi^{\frac{-1}{2}}\big\|_2  \|J_\pi\|_2 \|J_h\|_2.
\end{align}
Taking the supremum in $\bx$ and inserting the constraints $\|J_\pi\|_2\le L_\pi$, $\|J_h\|_2\le \bar J_h$, we obtain
\begin{align}
\big\|\sym(\mathcal{R}_{BJ})\big\|_2
\!\le\! \Big(\!2\sup_x \|M_\phi^{1/2} B M_\phi^{\frac{-1}{2}}\|_2\!\Big)\!\frac{L_\pi \bar J_h}{2}
\!=\! C_{BJ} L_\pi \bar J_h,\nonumber
\end{align}
where $C_{BJ}:=2\sup_x \|M_\phi^{1/2} B M_\phi^{\frac{-1}{2}}\|_2$.

\noindent\textbf{(iii) Metric time derivative part.}
Using submultiplicativity and $\|M_\phi^{\frac{-1}{2}}\|_2 = 1/\sqrt{\lambda_{\min}(M_\phi)} \le 1/\sqrt{m_{\min}}$,
\begin{align}
\Big\|M_\phi^{\frac{-1}{2}}\dot M_\phi M_\phi^{\frac{-1}{2}}\Big\|_2
&\le\; \|M_\phi^{\frac{-1}{2}}\|_2^2 \|\dot M_\phi\|_2
\le\; \frac{1}{m_{\min}}\;\|\dot M_\phi\|_2.\nonumber
\end{align}
Now apply the Frobenius spectral inequality and Cauchy Schwarz for tensors:
\begin{align}
\|\dot M_\phi\|_2
&\le \|\dot M_\phi\|_F
= \left\|\sum_{k} \frac{\partial M_\phi}{\partial x_k}  (f_{\mathrm{cl}})_k \right\|_F\le \Big\|\nabla M_\phi\Big\|_F   \|f_{\mathrm{cl}}\|_2\nonumber
\end{align}
By Assumption~\ref{ass:lipschitz}, $\|\nabla M_\phi\|_F \le L_M$ and
$\|f_{\mathrm{cl}}\| \le \|f\|+\|B\| \|\Pi^\mathrm{PD}_\theta\|
\le \bar f + \bar B \bar{u}$. Together,
\begin{align}
\big\|\mathcal{R}_{\dot M}\big\|_2
=\Big\|M_\phi^{\frac{-1}{2}}\dot M_\phi M_\phi^{\frac{-1}{2}}\Big\|_2
\le \frac{L_M}{m_{\min}}\big(\bar f+\bar B \bar{u}\big).
\end{align}
By the triangle inequality for the spectral norm,
\begin{align}
\big\|\!\sym(\hat{\mathcal{R}}(\bx))\big\|_2\!\!
\le \!C_f \!+\!C_{B\partial B}\!+\! C_{BJ} L_\pi \bar J_h\!+\!\!\frac{L_M}{m_{\min}}\!(\bar f\!+\!\bar B \bar{u}) \!-\! \alpha \nonumber
\end{align}
Since $\eigmax(\sym(\hat{\mathcal{R}}))=\|\sym(\hat{\mathcal{R}})\|_2$, this yields the desired upper bound.  A negative right-hand side implies $\eigmax(\sym(\mathcal{R}(\bx)))<0$ for all $\bx$, hence the contraction inequality holds.

\end{proof}
\begin{remark}
 If one prefers to avoid the factor $m_{\min}^{-1}$, define the normalized metric gradient budget
 \begin{align}
\widehat L_M \;:=\; \sup_x \big\| M_\phi(\bx)^{\frac{-1}{2}} (\nabla M_\phi(\bx)) M_\phi(\bx)^{\frac{-1}{2}}\big\|_F .
\end{align}
Then the same argument gives
$
\|M_\phi^{\frac{-1}{2}}\dot M_\phi M_\phi^{\frac{-1}{2}}\|_2
\le \widehat L_M \|f_{\mathrm{cl}}\|
\le \widehat L_M (\bar f+\bar B\bar{u}),
$
so the $\dot M$ contribution appears as $\widehat L_M(\bar f+\bar B \bar{u})$. \end{remark} 
\begin{corollary}[Lower bound for $\alpha$]
\label{cor:margin}
If $\norm{J_\pi(\bo)}_2 \le L_\pi$ and $\norm{\nabla M_\phi(\bx)}_F \le L_M$ with
\begin{align}
C_f + C_{B\partial B} + C_{BJ} L_\pi \bar J_h + \frac{L_M}{\sqrt{m_{\min}}}(\bar f+\bar B \bar{u}) \;<\; \alpha,
\label{eqn:alpha_lower_bound}
\end{align}
then \eqref{eq:contraction-ineq} holds. The residual bound \eqref{eq:PA-RB} shows how tightening network Lipschitz constraints improves the contraction margin.
\end{corollary}
\begin{remark}
{The bound in Theorem~\ref{thm:PA-lip} certifies contraction on the compact set $\mathcal{K}$ covered during training (and assumed forward-invariant). If the closed-loop trajectory leaves $\mathcal{K}$ at test time, the formal guarantee is not ensured. However, in practice, $\mathcal{K}$ can be enlarged to ensure broader feasibility.}
\end{remark}
\begin{figure*}[ht]
    \centering
    \subfloat[Wind speed $4.8$ $m/s$]{\resizebox{0.24\textwidth}{!}{\begin{tikzpicture}
\begin{axis}[
    grid=both,
    xlabel={$x$},
    ylabel={$y$},
    zlabel={$z$},
    view={60}{30}, 
        xmin=-1.0, xmax=1.0, ymin=-1.0, ymax=1.0,zmin=-0.5, zmax=0.5,
    axis equal image,
    legend pos=north east,
            label style={font=\Large},
    tick label style={font=\Large},
    legend style={font=\Large},
]

\pgfplotstableread[col sep=comma,header=false]{csv/contraction_10.csv}\datatable

\pgfplotstablegetrowsof{\datatable}
\pgfmathtruncatemacro{\nrows}{\pgfplotsretval-1}
\def\foundfirst{0}
\def\xzero{0}
\def\yzero{0}
\def\zzero{0}
\pgfplotsforeachungrouped \r in {0,...,\nrows} {%
  \pgfplotstablegetelem{\r}{[index]1}\of{\datatable}\edef\tempx{\pgfplotsretval}%
  \pgfplotstablegetelem{\r}{[index]2}\of{\datatable}\edef\tempy{\pgfplotsretval}%
  \pgfplotstablegetelem{\r}{[index]3}\of{\datatable}\edef\tempz{\pgfplotsretval}%
  \ifnum\foundfirst=0
    \ifx\tempx\empty\else
      \ifx\tempy\empty\else
        \ifx\tempz\empty\else
          \xdef\xzero{\tempx}%
          \xdef\yzero{\tempy}%
          \xdef\zzero{\tempz}%
          \def\foundfirst{1}%
        \fi
      \fi
    \fi
  \fi
}

\addplot3+[smooth, thick, mark=none, line width=0.5mm]
  table[
    x expr=\thisrowno{1}-\xzero,
    y expr=\thisrowno{2}-\yzero,
    z expr=\thisrowno{3}-\zzero,
            restrict expr to domain={\coordindex}{2:830},
  ] {\datatable};

\addplot3[
  surf,
  shader=interp,
  opacity=0.4,
  samples=20,
  domain=0:360,
  y domain=0:180,
  colormap/viridis,
] ({0.8*sin(y)*cos(x)}, {0.8*sin(y)*sin(x)}, {0.8*cos(y)});
\end{axis}

\end{tikzpicture}}}
    \hfil
    \subfloat[Wind speed $6.4$ $m/s$]{\resizebox{0.24\textwidth}{!}{\begin{tikzpicture}
\begin{axis}[
    grid=both,
    xlabel={$x$},
    ylabel={$y$},
    zlabel={$z$},
    view={60}{30}, 
        xmin=-1.0, xmax=1.0, ymin=-1.0, ymax=1.0,zmin=-0.5, zmax=0.5,
    axis equal image,
    legend pos=north east,
            label style={font=\Large},
    tick label style={font=\Large},
    legend style={font=\Large},
]

\pgfplotstableread[col sep=comma,header=false]{csv/contraction_30.csv}\datatable

\pgfplotstablegetrowsof{\datatable}
\pgfmathtruncatemacro{\nrows}{\pgfplotsretval-1}
\def\foundfirst{0}
\def\xzero{0}
\def\yzero{0}
\def\zzero{0}
\pgfplotsforeachungrouped \r in {0,...,\nrows} {%
  \pgfplotstablegetelem{\r}{[index]1}\of{\datatable}\edef\tempx{\pgfplotsretval}%
  \pgfplotstablegetelem{\r}{[index]2}\of{\datatable}\edef\tempy{\pgfplotsretval}%
  \pgfplotstablegetelem{\r}{[index]3}\of{\datatable}\edef\tempz{\pgfplotsretval}%
  \ifnum\foundfirst=0
    \ifx\tempx\empty\else
      \ifx\tempy\empty\else
        \ifx\tempz\empty\else
          \xdef\xzero{\tempx}%
          \xdef\yzero{\tempy}%
          \xdef\zzero{\tempz}%
          \def\foundfirst{1}%
        \fi
      \fi
    \fi
  \fi
}

\addplot3+[smooth, thick, mark=none, line width=0.5mm]
  table[
    x expr=\thisrowno{1}-\xzero,
    y expr=\thisrowno{2}-\yzero,
    z expr=\thisrowno{3}-\zzero,
            restrict expr to domain={\coordindex}{2:830},
  ] {\datatable};

\addplot3[
  surf,
  shader=interp,
  opacity=0.4,
  samples=20,
  domain=0:360,
  y domain=0:180,
  colormap/viridis,
] ({0.8*sin(y)*cos(x)}, {0.8*sin(y)*sin(x)}, {0.8*cos(y)});
\end{axis}
\end{tikzpicture}}}
    \hfil
    \subfloat[Wind speed $8.0$ $m/s$]{\resizebox{0.24\textwidth}{!}{\begin{tikzpicture}
\begin{axis}[
    grid=both,
    xlabel={$x$},
    ylabel={$y$},
    zlabel={$z$},
    view={60}{30}, 
        xmin=-1.0, xmax=1.0, ymin=-1.0, ymax=1.0,zmin=-0.5, zmax=0.5,
    axis equal image,
    legend pos=north east,
            label style={font=\Large},
    tick label style={font=\Large},
    legend style={font=\Large},
]

\pgfplotstableread[col sep=comma,header=false]{csv/contraction_40.csv}\datatable

\pgfplotstablegetrowsof{\datatable}
\pgfmathtruncatemacro{\nrows}{\pgfplotsretval-1}
\def\foundfirst{0}
\def\xzero{0}
\def\yzero{0}
\def\zzero{0}
\pgfplotsforeachungrouped \r in {0,...,\nrows} {%
  \pgfplotstablegetelem{\r}{[index]1}\of{\datatable}\edef\tempx{\pgfplotsretval}%
  \pgfplotstablegetelem{\r}{[index]2}\of{\datatable}\edef\tempy{\pgfplotsretval}%
  \pgfplotstablegetelem{\r}{[index]3}\of{\datatable}\edef\tempz{\pgfplotsretval}%
  \ifnum\foundfirst=0
    \ifx\tempx\empty\else
      \ifx\tempy\empty\else
        \ifx\tempz\empty\else
          \xdef\xzero{\tempx}%
          \xdef\yzero{\tempy}%
          \xdef\zzero{\tempz}%
          \def\foundfirst{1}%
        \fi
      \fi
    \fi
  \fi
}

\addplot3+[smooth, thick, mark=none, line width=0.5mm]
  table[
    x expr=\thisrowno{1}-\xzero,
    y expr=\thisrowno{2}-\yzero,
    z expr=\thisrowno{3}-\zzero,
            restrict expr to domain={\coordindex}{2:830},
  ] {\datatable};

\addplot3[
  surf,
  shader=interp,
  opacity=0.4,
  samples=20,
  domain=0:360,
  y domain=0:180,
  colormap/viridis,
] ({0.8*sin(y)*cos(x)}, {0.8*sin(y)*sin(x)}, {0.8*cos(y)});
\end{axis}
\end{tikzpicture}}}
    \hfil
    \subfloat[Wind speed $9.6$ $m/s$]{\resizebox{0.24\textwidth}{!}{\begin{tikzpicture}
\begin{axis}[
    grid=both,
    xlabel={$x$},
    ylabel={$y$},
    zlabel={$z$},
    view={60}{30}, 
        xmin=-1.0, xmax=1.0, ymin=-1.0, ymax=1.0,zmin=-0.5, zmax=0.5,
    axis equal image,
    legend pos=north east,
            label style={font=\Large},
    tick label style={font=\Large},
    legend style={font=\Large},
]

\pgfplotstableread[col sep=comma,header=false]{csv/contraction_60.csv}\datatable

\pgfplotstablegetrowsof{\datatable}
\pgfmathtruncatemacro{\nrows}{\pgfplotsretval-1}
\def\foundfirst{0}
\def\xzero{0}
\def\yzero{0}
\def\zzero{0}
\pgfplotsforeachungrouped \r in {0,...,\nrows} {%
  \pgfplotstablegetelem{\r}{[index]1}\of{\datatable}\edef\tempx{\pgfplotsretval}%
  \pgfplotstablegetelem{\r}{[index]2}\of{\datatable}\edef\tempy{\pgfplotsretval}%
  \pgfplotstablegetelem{\r}{[index]3}\of{\datatable}\edef\tempz{\pgfplotsretval}%
  \ifnum\foundfirst=0
    \ifx\tempx\empty\else
      \ifx\tempy\empty\else
        \ifx\tempz\empty\else
          \xdef\xzero{\tempx}%
          \xdef\yzero{\tempy}%
          \xdef\zzero{\tempz}%
          \def\foundfirst{1}%
        \fi
      \fi
    \fi
  \fi
}

\addplot3+[smooth, thick, mark=none, line width=0.5mm]
  table[
    x expr=\thisrowno{1}-\xzero,
    y expr=\thisrowno{2}-\yzero,
    z expr=\thisrowno{3}-\zzero,
            restrict expr to domain={\coordindex}{2:700},
  ] {\datatable};

\addplot3[
  surf,
  shader=interp,
  opacity=0.4,
  samples=20,
  domain=0:360,
  y domain=0:180,
  colormap/viridis,
] ({0.8*sin(y)*cos(x)}, {0.8*sin(y)*sin(x)}, {0.8*cos(y)});
\end{axis}
\end{tikzpicture}}}
    \caption{\footnotesize Trajectories of the quadruped performing a handstand under wind disturbances of increasing magnitude. Even though
    \ours was never trained on these severe wind disturbances, it preserves boundedness and maintains steady posture after each perturbation (wind speeds).
    }
    \label{fig:tracking_error_wind}
\end{figure*}
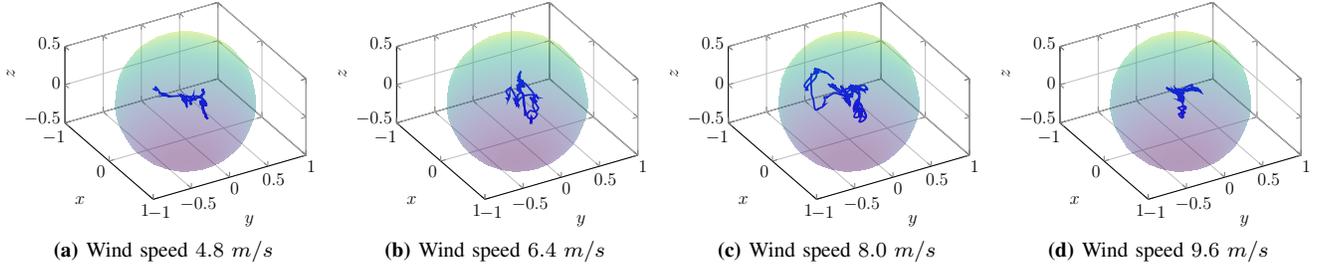

\begin{lemma}[Lower bound on $\epsilon_\alpha$]
Let $V_\phi(\bx,t)=\boldsymbol{e}^\top M_\phi(\bx)\boldsymbol{e}$ with $m_{\min}I\preceq M_\phi(\bx)\preceq m_{\max}I$, and define $\mathcal{R}(\bx)$ as in \eqref{eq:PA-Lcontr}.
Assume the Lipschitz bounds of Assumption~\ref{ass:lipschitz} hold, so that \eqref{eq:PA-RB} holds.
Then for all $\bx,\boldsymbol{e}$,
\begin{align}
\dot V_\phi(\bx,t)+\alpha V_\phi(\bx,t)\ \le\ \Xi V_\phi(\bx,t).
\end{align}
In particular, no state independent constant $\epsilon_\alpha>0$ can be guaranteed to satisfy $\dot V+\alpha V\le -\epsilon_\alpha$ for all $\boldsymbol{e}\neq 0$. However, if $\epsilon_\alpha$ is chosen state-dependent as $\epsilon_\alpha(\bx,\boldsymbol{e})=\underline{\epsilon_\alpha} V_\phi(\bx,t)$, then any $\underline{\epsilon_\alpha}\in(0,\alpha-\Xi]$ guarantees $\dot V+\alpha V\le -\epsilon_\alpha(\bx,\boldsymbol{e})$ for all $\bx,\boldsymbol{e}$.
\label{lemma:epsilon}
\end{lemma}
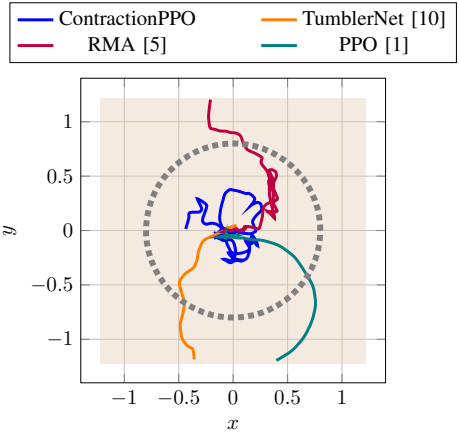
\begin{figure}[ht]
    \centering
    {\resizebox{0.35\textwidth}{!}{\begin{tikzpicture}
\begin{axis}[
scale=0.9,
    grid=both,
    xlabel={$x$}, ylabel={$y$},
    xmin=-1.4, xmax=1.4, ymin=-1.4, ymax=1.4,
    axis equal image,
    legend pos=north east,
    label style={font=\normalsize},
    tick label style={font=\normalsize},
    legend style={
        at={(0.5,1.05)},
        anchor=south,
        legend columns=2,
        font=\normalsize,
        /tikz/every even column/.append style={column sep=1cm}
    },
]

\addplot [draw=none, fill=brown!40, opacity=0.4, forget plot] 
  coordinates {(-1.22,-1.22) (1.22,-1.22) (1.22,1.22) (-1.22,1.22)} -- cycle;

\pgfplotstableread[col sep=comma,header=false]{csv/contraction_0_v2.csv}\datatableA
\pgfplotstablegetelem{0}{[index]1}\of{\datatableA} \pgfmathsetmacro{\xzeroA}{\pgfplotsretval}
\pgfplotstablegetelem{0}{[index]3}\of{\datatableA} \pgfmathsetmacro{\yzeroA}{\pgfplotsretval}
\addplot+[smooth, thick, mark=none, line width=0.5mm, blue] 
  table[x expr=\thisrowno{1}-\xzeroA,
        y expr=\thisrowno{3}-\yzeroA,
        restrict expr to domain={\coordindex}{200:330}] {\datatableA};
\addlegendentry{\ours}

\pgfplotstableread[col sep=comma,header=false]{csv/nature_paper.csv}\datatableB
\pgfplotstablegetelem{0}{[index]1}\of{\datatableB} \pgfmathsetmacro{\xzeroB}{\pgfplotsretval}
\pgfplotstablegetelem{0}{[index]3}\of{\datatableB} \pgfmathsetmacro{\yzeroB}{\pgfplotsretval}
\addplot+[smooth, thick, mark=none, line width=0.5mm, orange] 
  table[x expr=\thisrowno{1}-\xzeroB,
        y expr=\thisrowno{3}-\yzeroB,
        restrict expr to domain={\coordindex}{80:97}] {\datatableB};
\addlegendentry{TumblerNet \cite{xiao2025learning}}

\pgfplotstableread[col sep=comma,header=false]{csv/rma_v2.csv}\datatableC
\pgfplotstablegetelem{0}{[index]1}\of{\datatableC} \pgfmathsetmacro{\xzeroC}{\pgfplotsretval}
\pgfplotstablegetelem{0}{[index]3}\of{\datatableC} \pgfmathsetmacro{\yzeroC}{\pgfplotsretval}
\addplot+[smooth, thick, mark=none, line width=0.5mm, purple] 
  table[x expr=\thisrowno{1}-\xzeroC,
        y expr=\thisrowno{3}-\yzeroC,
        restrict expr to domain={\coordindex}{350:431}] {\datatableC};
\addlegendentry{RMA \cite{kumar2022adapting_rma_limit1}}

\pgfplotstableread[col sep=comma,header=false]{csv/ppo.csv}\datatableD
\pgfplotstablegetelem{0}{[index]1}\of{\datatableD} \pgfmathsetmacro{\xzeroD}{\pgfplotsretval}
\pgfplotstablegetelem{0}{[index]3}\of{\datatableD} \pgfmathsetmacro{\yzeroD}{\pgfplotsretval}
\addplot+[smooth, thick, mark=none, line width=0.5mm, teal] 
  table[x expr=\thisrowno{1}-\xzeroD,
        y expr=\thisrowno{3}-\yzeroD,
        restrict expr to domain={\coordindex}{80:167}] {\datatableD};
\addlegendentry{PPO \cite{schulman2017proximal_ppo}}

\draw[gray, dotted, thick, line width=1mm] (axis cs:0,0) circle [radius=0.8];

\end{axis}
\end{tikzpicture}}}
    \caption{\footnotesize Tracking error (in m) of the quadruped handstand for ContractionPPO, TumblerNet \cite{xiao2025learning}, RMA \cite{kumar2022adapting_rma_limit1}, and PPO \cite{schulman2017proximal_ppo} on a moving platform (marked in light brown). While all baselines exhibit drift outside the feasible region, \ours maintains bounded trajectories with consistently lower RMS tracking error, demonstrating certified incremental stability.}
    \label{fig:tracking_error}
\end{figure}

\begin{proof}
In particular, no state independent constant $\epsilon_\alpha>0$ can be guaranteed to satisfy $\dot V+\alpha V\le -\epsilon_\alpha$ for all $\boldsymbol{e}\neq 0$. However, if $\epsilon_\alpha$ is chosen state-dependent as $\epsilon_\alpha(\bx,\boldsymbol{e})=\underline{\epsilon_\alpha} V_\phi(\bx,t)$, then any $\underline{\epsilon_\alpha}\in(0,\alpha-\Xi]$ guarantees $\dot V+\alpha V\le -\epsilon_\alpha(\bx,\boldsymbol{e})$ for all $\bx,\boldsymbol{e}$.    
\end{proof}
\begin{theorem}[Robustness]\label{thm:PA-robust}
Consider the dynamics \eqref{eq:dynamics} with the learned contraction metric $M_\phi$ satisfying the conditions in Theorem \ref{thm:PA-lip}. Under Assumption~\ref{ass:lipschitz}, if $\boldsymbol{e}^\top\mathcal{R}(\bx)\boldsymbol{e}\le 0$ on $\mathcal{K}$, then for any two solutions $\bx,\bx_d$ driven by disturbance $\boldsymbol{d}$ with $\norm{\boldsymbol{d}}\le \bar d$, we have
\begin{align}
\norm{\boldsymbol{e}(t)}\le \frac{V(\bx,0)}{\sqrt{{m}_{\min}}} e^{-\alpha t}+\frac{\bar{d}}{\alpha} \sqrt{\chi}\left(1-e^{-\alpha t}\right)
\label{eq:PA-iss}
\end{align}
where $\boldsymbol{e}(t)=\bx(t)-\bx_d(t)$ and $\chi=m_{\max}/m_{\min}$.
\end{theorem}
\begin{proof}
Please see the proof of Theorem 3.1 in \cite{tsukamoto2021contraction_tutorial}.
\end{proof}
\begin{remark}
{If $A_{\mathrm{cl}}(\bx)$ or $\dot M_\phi(\bx)$ are computed approximately, the resulting error acts as an additional bounded residual in the contraction condition. This mainly reduces the certified margin (e.g., smaller $\alpha$ or larger $\epsilon_\alpha$) and degrades gracefully under bounded mismatch, consistent with Theorem~\ref{thm:PA-robust}.}
\end{remark}
To summarize, Theorem \ref{thm:PA-lip} ensures that contraction-based stability verified in simulation extends to real-world deployments, even in the presence of modeling errors or sim-to-real gaps. Theorem \ref{thm:PA-robust} complements this by proving that once contraction holds, the system retains incremental exponential stability under bounded external disturbances. Together, these results establish strong theoretical guarantees for certified RL policies that are both robust to real-world uncertainty and external perturbations.


\section{Results\label{sec:results}}
In this section, we benchmark our \ours controller against recent baselines \cite{xiao2025learning,schulman2017proximal_ppo,kumar2022adapting_rma_limit1} in both simulation and hardware and structure the discussion around the following questions $(i)$ how does its locomotion performance (eg. tracking error) compare with a PPO \cite{schulman2017proximal_ppo} and recent baseline algorithms \cite{xiao2025learning,kumar2022adapting_rma_limit1} on a moving platform, 
$(ii)$ how robust is \ours to external disturbances, i.e., pushes, impulsive forces or platform vibrations, compared to other baseline algorithms \cite{kumar2022adapting_rma_limit1,xiao2025learning,schulman2017proximal_ppo}. To access the performance of ContractionPPO, we conducted all training in simulation using the IsaacLab framework with the Unitree Go1 quadruped and real-world experiments on a moving platform in the CAST Arena at Caltech.
{Reflective markers were fixed to the quadruped, and an OptiTrack motion capture system was utilized to record the robot’s planar position ($x$ and $y$) throughout the hardware experiments}. Experiments were run on a desktop equipped with an Intel Core i7-13700F CPU, 32 GB RAM, and an NVIDIA RTX 3080 GPU. Leveraging 4,096 agents in parallel, our algorithm completed full training within approximately two hours of wall-clock time. For the PD controller that ran at $200$ Hz, the $K_p$ and $K_d$ gains were set to 30 and 0.8, respectively.

We adopt the RSL-RL implementation of PPO \cite{rudin2022learning} for the reinforcement learning layer (i.e., the PPO layer in Fig.~\ref{fig:approach}) of our approach. The contraction layer (green region in Fig.~\ref{fig:approach}) is parameterized using MLPs with two hidden layers. To ensure Lipschitz continuity, we apply spectral normalization to the weights of the contraction layer MLPs as well as the PPO policy. The training of both these networks is performed using the Adam optimizer. Table~\ref{tab:hyperparams} reports all PPO, and contraction hyperparameters.
\begin{table}[H]
\centering
\caption{\footnotesize Key hyperparameters}
\label{tab:hyperparams}
\begin{tabular}{|l|c|}
\hline
\textbf{Parameter} & \textbf{Value} \\
\hline
Actor/Critic Net. & MLP (512, 256, 128) \\
Learning Rate and activation & $1\times10^{-3}$ and ELU \\
Epochs per update & 5\\
Max. iterations & 15000\\
Rollout Length and $\gamma$ & 24 steps and 0.99 \\
Mini-batches per epoch & 4\\
\hline
$M_\phi$ and $w_{\text{contr}}$  &  (128, 64) and 0.01 \\
\hline
Parallel Environments & 4096 \\
\hline
\end{tabular}
\end{table}

\begin{figure}[]
    \centering
\includegraphics[width=0.45\textwidth]{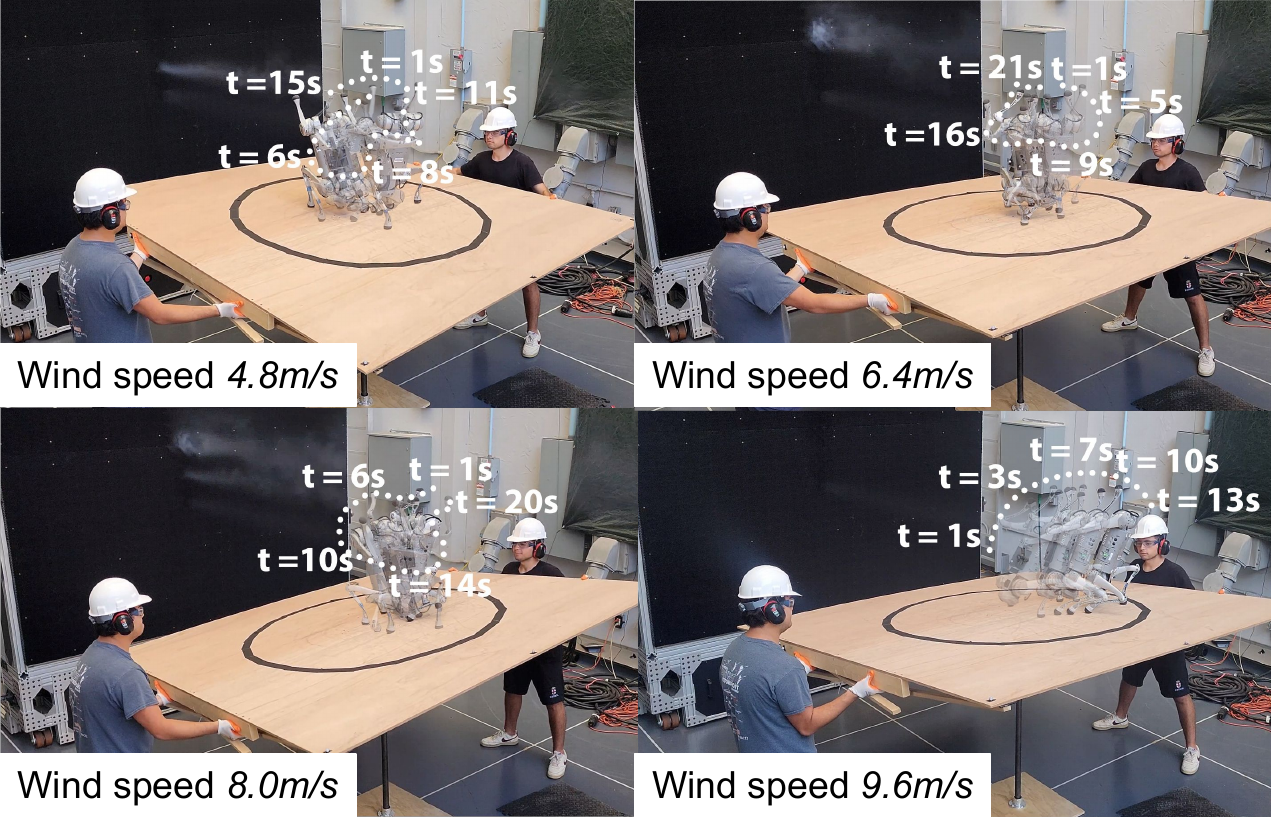}
    \caption{\footnotesize Quadruped trajectories during handstand under wind disturbances of varying magnitudes. As the disturbance intensity increases, transient deviations from the initial position become more evident. Note that \ours was never trained on external disturbances. Despite this, \ours policy ensures bounded trajectories and lies inside the circle marked by the black curve. 
    }
    \label{fig:}
    \vspace{-0.5cm}
\end{figure}
\begin{figure}
    \centering
\includegraphics[width=0.9\linewidth]{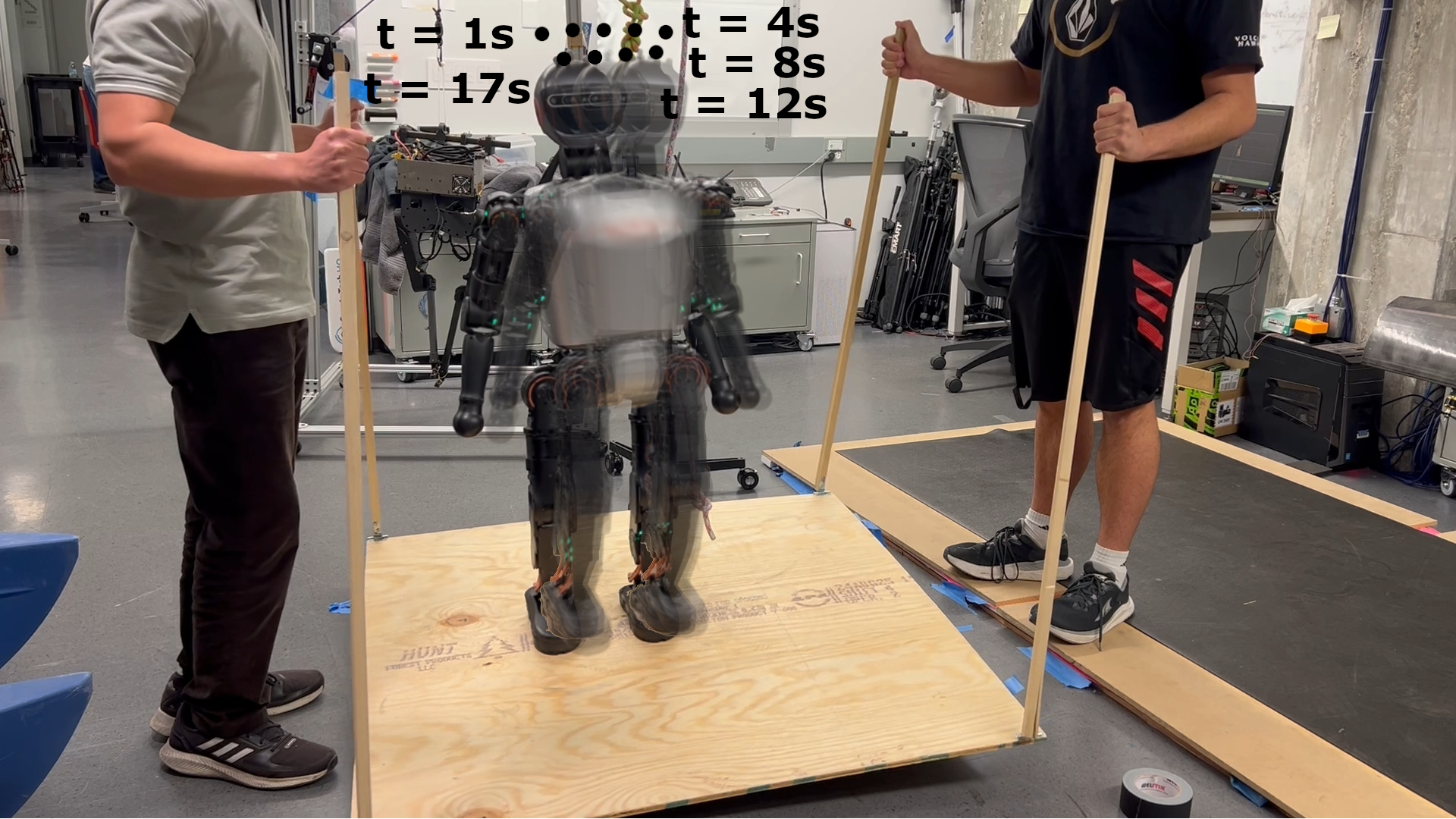}
    \caption{{Stabilization of Booster T1 on moving platform using \ours.}}
    \label{fig:t1_stabilization}
\end{figure}
\subsection{Locomotion Performance}
We evaluate the performance by initializing the quadruped at the center of a circle of radius $0.8m$ while commanding a handstand while remaining within that region. \ours consistently keeps the quadruped within the circle. However, PPO \cite{schulman2017proximal_ppo} and TumblerNet \cite{xiao2025learning} drift outside this boundary (see Fig.~\ref{fig:tracking_error}), and RMA, while stronger than other baselines, ultimately fails to remain inside. The qualitative difference is most evident in long rollouts i.e., \ours maintains bounded trajectories that do not go out of the circular region, while the baselines exhibit slow drift followed by loss of balance once the boundary is crossed (see Fig.~\ref{fig:ppo_vs_contraction_handstand}). {Beyond the quadruped handstand benchmark, we also evaluate ContractionPPO on a Booster T1 humanoid stabilization task on a moving platform (see Fig. \ref{fig:t1_stabilization}). This experiment demonstrates that our approach extends to a different robot morphology, while retaining robust stabilization performance.}

Table~\ref{table:comparison} reports failure rates across increasing control points and combined episodes. {The control points define the trajectory used to generate B-spline motions for the platform. Increasing the number of control points results in faster and more aggressive movements, whereas fewer points yield slower motions but sustain steeper angles for longer durations, thereby increasing task difficulty.}  \ours significantly outperforms all baselines, maintaining a very low average failure rate of 10\%, with only a gradual increase as the number of control points grows. In stark contrast, TumblerNet \cite{xiao2025learning} fails consistently (100\%) across all settings, indicating poor generalization or lack of robustness. RMA \cite{kumar2022adapting_rma_limit1} shows moderate performance (38\%), while PPO \cite{schulman2017proximal_ppo} performs poorly (89\%), likely due to the absence of formal stability enforcement. These results underscore the effectiveness of incorporating contraction-based certification during training, yielding robust and certifiably stable behavior even as control complexity increases. In addition, the learned $M_\phi$ increases the span over which certified feedback controllers can be applied. In contrast, as seen from Table \ref{table:comparison} a fixed contraction such as $M_\phi$ results in poor performance.

\ours consistently outperforms baselines such as PPO 
\cite{schulman2017proximal_ppo}, RMA \cite{kumar2022adapting_rma_limit1}, and TumblerNet \cite{xiao2025learning} due to its explicit enforcement of incremental stability via contraction theory during training. While standard RL policies \cite{schulman2017proximal_ppo} may overfit to reward signals without ensuring safe or bounded behavior under perturbations, \ours introduces a learned contraction metric that certifies exponential convergence of trajectories. This not only guarantees recovery after external disturbances but also ensures stability under model mismatch and noisy observations. In contrast to approaches like RMA \cite{kumar2022adapting_rma_limit1} or TumblerNet \cite{xiao2025learning}, which lack formal guarantees or require separate adaptation mechanisms, our method integrates stability directly into the learning process. The contraction metric acts as an implicit regularizer, shaping the policy space toward robust and stable solutions. Importantly, our framework also accounts for the approximation error between the value function in simulation and the one experienced at deployment through the inclusion of a conservative stability margin $\epsilon_\alpha$. This ensures that the contraction condition continues to hold even in the presence of sim-to-real gaps.

\subsection{Robustness}
To prove robustness, we apply wind disturbances using the wind tunnel and repeat the handstand experiments using \ours at speeds of 4.8, 6.4, 8.0, and 9.6 $m/s$. Notably, \ours was never trained under such high-magnitude perturbations. Despite this, it maintains the handstand within the $0.8$m circle, with transient deviations that quickly converge within the circle. The trajectory plots in Fig.~\ref{fig:tracking_error_wind} show that \ours preserves boundedness and returns to steady posture after impulses, consistent with the incremental stability guarantee (Theorem~\ref{thm:PA-robust}). These results underscore that enforcing contraction conditions during end-to-end training leads to controllers that not only perform well under nominal conditions but also maintain strong stability margins under substantial disturbances.


\begin{table}[t]
\centering
\caption{\footnotesize Comparison of methods across control points and violations. Numbers represent failure ratio over all episodes. Combined is $2500$ episodes, and each control point is $500$. Lower is better. 
}
\begin{tabular}{|l|c|c|c|c|c|c|}
\hline
\multirow{2}{*}{{Method}} 
  & \multicolumn{5}{c|}{{Control Points}} 
  & \multirow{2}{*}{{Combined}} \\ \cline{2-6}
  & 10 & 20 & 30 & 40 & 50 & \\
\hline
ContractionPPO & \textbf{0.00} & \textbf{0.00} & \textbf{0.00} & \textbf{0.00} & \textbf{0.00} & \textbf{0.00} \\
\hline
TumblerNet \cite{xiao2025learning} & 1.00 & 1.00 & 1.00 & 1.00 & 1.00 & 1.00 \\
\hline
RMA \cite{kumar2022adapting_rma_limit1} & 0.23 & 0.23 & 0.19 & 0.19 & 0.16 & 0.20 \\
\hline
PPO \cite{schulman2017proximal_ppo} & 0.98 & 0.98 & 0.98 & 0.99 & 0.99 & 0.99 \\
\hline
$M_\phi=I$ & 1.00 & 1.00 & 1.00 & 1.00 & 1.00 & 1.00 \\
\hline
\end{tabular}
\label{table:comparison}
\end{table}

\section{Conclusions\label{sec:conclusion}}
We present ContractionPPO, a framework that augments PPO with a differentiable contraction metric layer to certify incremental stability in quadruped locomotion. By training the contraction metric using privileged state while keeping the policy observation-based, our method ensures deployable guarantees. Despite not being trained under strong disturbances like wind, \ours maintains stability at test time, highlighting a shift from high performing RL to certified RL. Hardware experiments confirm that \ours achieves robust, provably stable control under both nominal and perturbed conditions. Future work includes integrating this framework with fault tolerant systems.

\bibliography{main.bib}





\end{document}